\documentclass{article}

\usepackage{amsmath, mathrsfs, amssymb, amsthm, mathtools, bbm, bm} 

\usepackage{algorithm} 
\usepackage{algcompatible} 

\usepackage{makecell} 

\usepackage[english]{babel}
\usepackage[utf8]{inputenc}
\usepackage[T1]{fontenc}
\usepackage[a4paper,left=3cm,right=3cm,top=3cm,bottom=3.2cm]{geometry} 
\usepackage{tikz}
\usetikzlibrary[topaths] 

\newtheorem{theorem}{Theorem}
 
\newtheorem{definition}{Definition}
\newtheorem{lemma}{Lemma}
\newtheorem{remark}{Remark}

\newtheorem{proposition}{Proposition}
\newtheorem{assumption}{Assumption}

\newtheorem*{theorem*}{Theorem}
\newtheorem*{example*}{Example} 
\newtheorem*{definition*}{Definition}
\newtheorem*{lemma*}{Lemma}
\newtheorem*{remark*}{Remark}
\newtheorem*{corollary*}{Corollary}
\newtheorem*{proposition*}{Proposition}
\newtheorem*{assumption*}{Assumption}
\newtheorem*{claim*}{Claim}

\newtheoremstyle{TheoremNum}
        {\topsep}{\topsep}              
        {\itshape}                      
        {}                              
        {\bfseries}                     
        {.}                             
        { }                             
        {\thmname{#1}\thmnote{ \bfseries #3}}
\theoremstyle{TheoremNum}

\newtheoremstyle{LemmaNum}
        {\topsep}{\topsep}              
        {\itshape}                      
        {}                              
        {\bfseries}                     
        {.}                             
        { }                             
        {\thmname{#1}\thmnote{ \bfseries #3}}
\theoremstyle{LemmaNum}

\usepackage{hyperref} 

\usepackage{natbib, float}

\DeclareMathAlphabet{\mathpzc}{OT1}{pzc}{m}{it}

\newcommand{\Ot}{ \widetilde{\mathcal{O}} }

\renewcommand{\L}{ \mathcal{L} }

\newcommand{\Q}{ \mathbb{Q} }

\newcommand{\clip}{ \text{clip} }

\newcommand{\wt}[1]{\widetilde{#1}}
\newcommand{\wh}[1]{\widehat{#1}}

\newcommand{\E}{ \mathbb{E} }
\renewcommand{\Pr}{ \mathbb{P} }

\newcommand{\J}{ \mathfrak{J} }

\renewcommand{\[}{ \left[ }
\renewcommand{\]}{ \right] }

\renewcommand{\(}{ \left( }
\renewcommand{\)}{ \right) }

\newcommand{\I}{ \mathbb{I} }

\renewcommand{\P}{ \mathcal{P} }
\renewcommand{\O}{ \mathcal{O} }

\newcommand{\advB}{ \frac{ \log \frac{(1-\rho) \epsilon }{KT} }{ \log \rho } }

\newcommand{\Reg}{ \mathrm{Reg} } 

\newcommand{\Regadv}{ \mathrm{Reg}^{\mathrm{adv}} } 

\newcommand{\xit}{  \max \left\{ 1 + \frac{\rho}{(1 - \rho)^2 } , \frac{ \log (1 - \rho)  }{ \log \rho } + \frac{ 5 \log  t  }{ \log 1 / \rho } \right\} }

\newcommand{ \wtC }[1]{ \sqrt{ \frac{  8 \xi_t \log t }{ {#1} } }   }

\begin{document} 

\title{Towards Fundamental Limits of Multi-armed Bandits with Random Walk Feedback} 

\author{Tianyu Wang \qquad Lin F. Yang \qquad Zizhuo Wang}  

\date{} 

\maketitle

\begin{abstract} 
    In this paper, we consider a new Multi-Armed Bandit (MAB) problem where arms are nodes in an unknown and possibly changing graph, and the agent (i) initiates random walks over the graph by pulling arms, (ii) observes the random walk trajectories, and (iii) receives rewards equal to the lengths of the walks. We provide a comprehensive understanding of this problem by studying both the stochastic and the adversarial setting. We show that this problem is not easier than a standard MAB in an information theoretical sense, although additional information is available through random walk trajectories. Behaviors of bandit algorithms on this problem are also studied. 
\end{abstract} 



\section{Introduction} 

Multi-Armed Bandit (MAB) problems simultaneously call for exploitation of good options and exploration of the decision space. 
Algorithms for this problem find applications in various domains, from medical trials \citep{robbins1952some} to online advertisement \citep{li2010contextual}. Many authors have studied bandit problems from different perspectives. 

In this paper, we study a new bandit learning problem where the feedback is depicted by a random walk over the arms. That is, each time an arm/node $i$ is played, one observes a random walk over the arms/nodes from $i$ to an absorbing node, and the reward/loss is the length of this random walk. 
In this learning setting, we want to carefully select which nodes to initialize the random walks, so that the hitting time to the absorbing node is maximized/minimized. 
This learning protocol captures important problems in computational social networks. In particular, our learning protocol encapsulates an online learning version of the influence maximization problem \citep{kempe2003maximizing}. See e.g., \citep{arora2017debunking} for a survey on influence maximization. The goal of influence maximization is to find a node so that a diffusion starting from that node can propagate through the network and can influence as many nodes as possible. Our problem provides an online learning formulation of the influence maximization problem, where the diffusion over the graph is modeled as a random walk over the graph. As a concrete example, the word-of-mouth rating of a movie on social networks can be captured by the influence maximization model \citep{arora2017debunking}. 



More formally, we consider the following model. 
The environment is modeled by a graph $G = \(V,E \)$, where $V$ consists of transient nodes $[K] := \{1, 2, \dots, K\}$ and an absorbing node $*$. Each edge $ij$ ($i \in [K], j \in [K] \cup \{ * \}$) can encode two quantities, a transition probability from $i$ to $j$ and the distance from $i$ to $j$.
For $t = 1,2, \dots, T$, we pick a node to start a random walk, and observe the random walk trajectory from the selected node to the absorbing node. 
For each random walk, we use its hitting time (to the absorbing node $*$) to model how long-lasting it is. With this formulation, we can define a bandit learning problem for the question. Each time, the agent picks a node in $G$ to start a random walk, observes the trajectory, and receives the hitting time of the random walk as reward. In this setting, the performance of learning algorithms is typically measured by \textit{regret}, which is the difference between the rewards of an optimal node and the rewards of the nodes played by the algorithm. 
Unlike standard multi-armed bandit problems, the feedback is random walk trajectories and thus reveals information not only about the node played, but the environment (transitions/distances among nodes) as well. 



Interestingly, the extra information from the random walk does not trivialize the learning in a mini-max sense. Intuitively, the sample/event space describes how much information the feedback can carry. If we execute a policy $\pi$ for $T$ epochs on a problem instance, 
the sample space is then $ \( \cup_{h=1}^\infty \mathcal{B}^h  \)^T $, where $ \mathcal{B} $ is the space of all outcomes that a single step on a trajectory can generate. 
For example, if all edge length can be sampled from $[0,1]$, then $ \mathcal{B} = [0,1] \times [K]$, since a single step on a trajectory might be any node, and the corresponding edge length may be any number from $[0,1]$. In this case the sample space of a simple epoch is $\cup_{h=1}^\infty \mathcal{B}^h $, where the union up to infinity captures the fact that the trajectory can be arbitrarily long. Indeed, this set $ \( \cup_{h=1}^\infty \( [0,1] \times [K] \)^h  \)^T  $ contains much richer information than that of standard MAB problems, which is $ [0,1]^T $ for $T$ rounds of pulls. This richer sample space means that: {each feedback carries much more information and thus our problem can be strictly easier than a standard MAB.} 
However, the information theoretical lower bounds for our problem are of order $\wt{\Omega} \( \sqrt{T} \)$. In other words, we prove that even though each trajectory carries much more information than a reward sample (and has a chance of revealing all information about the environment when the trajectory is long), no algorithm can beat the bound $\wt{\Omega} \( \sqrt{T} \)$ in a mini-max sense. 



In summary, the contributions of our paper is as follows. 
\begin{enumerate} 
    \item We propose a new online learning problem that is compatible with important computational social network problem including influence maximization, as already discussed in the introduction. 
    \item We prove lower bounds for this problem in a random walk trajectories sample space. Our results show that, although random walk trajectories carry much more information than reward samples, the additional information does not simplify the problem. More specifically, no algorithm can beat an $ \widetilde{\Omega} \( \sqrt{T} \) $ lower bound in a mini-max sense. These information theoretical findings of the newly proposed problem are discussed in Section \ref{sec:info}. 
    \item We propose algorithms for the bandit problems with random walk feedback, and show that the performance of our algorithms improves over that of the standard MAB algorithms. 
\end{enumerate}
\subsection{Related Works} 

Bandit problems date its history back to at least \citet{thompson1933likelihood}, and have been studied extensively in the literature. One of the the most popular approaches to the stochastic bandit problem is the Upper Confidence Bound (UCB) algorithms \citep{robbins1952some, lai1985asymptotically, auer2002using}. Various extensions of UCB algorithms have been studied \citep{srinivas2009gaussian, abbasi2011improved, agrawal2012analysis, bubeck2012best, seldin2014one}. Specifically, some works use KL-divergence to construct the confidence bound \citep{lai1985asymptotically, garivier2011kl, maillard2011finite}, or include variance estimates within the confidence bound  \citep{audibert2009exploration, auer2010ucb}. UCB is also used in the contextual learning setting  \citep[e.g.,][]{li2010contextual, krause2011contextual, slivkins2014contextual}. The UCB algorithm and its variations are also used for other feedback settings, including the stochastic combinatorial bandit problem \citep{chen2013combinatorial,chen2016combinatorial}. 
Parallel to the stochastic setting, studies on the adversarial bandit problem form another line of literature. Since randomized weighted majorities \citep{littlestone1994weighted}, exponential weights remains a top strategy for adversarial bandits \citep{auer1995gambling,cesa1997use, auer2002nonstochastic}. 
Many efforts have been made to improve/extend exponential weights algorithms. For example, \citet{kocak2014efficient} target at implicit variance reduction. \citet{mannor2011bandits,alon2013bandits} study a partially observable setting. Despite the large body of literature, no previous work has, to the best of our knowledge, explicitly focused on problems where the feedback is a random walk. 

For both stochastic bandits and adversarial bandits, lower bounds in different scenarios have been derived,  since the $\mathcal{O} (\log T)$ asymptotic lower bounds for consistent policies \citep{lai1985asymptotically}. Worst case bound of order ${\mathcal{O}} (\sqrt{T})$ have also been derived \citep{auer1995gambling} for the stochastic setting. In addition to the classic stochastic setting, lower bounds in other stochastic (or stochastic-like) settings have also been considered, including PAC-learning complexity \citep{mannor2004sample}, best arm identification complexity \citep{kaufmann2016complexity, chen2017nearly}, and lower bounds in continuous spaces \citep{kleinberg2008multi}. Lower bound problems for adversarial bandits may be converted to lower bound problems for stochastic bandits  \citep{auer1995gambling} in many cases. Yet the above mentioned works do not cover the lower bounds for our settings. 

The Stochastic Shortest Path (with adversarial edge length) \citep[e.g.,][]{bertsekas1991analysis,neu2012adversarial, rosenberg2020adversarial} and the online MDP problems \citep[e.g.,][]{even2009online,gergely2010online,dick2014online,jin2019learning} are related to our problem. However, these settings are fundamentally different because of the sample space generated by the possible trajectories. In all previously studied settings, a control is available at each step, and a regret is immediately incurred. In this regard, the infinite-length free trajectory scenario is impossible in all previous studies. In other words, every trajectory in previous works is effectively of length one, since a control is imposed, and a regret is incurred every time the state changes. 



\section{Problem Setting} 
The learning process repeats for $T$ epochs and the learning environment is described by graphs $ G_1, G_2, \dots, G_T $ for epochs $t = 1,2,\dots,T$. The graph $G_t$ is defined on $K$ transient nodes $ [K] = \{ 1,2,\dots, K \}$ and one absorbing node denoted by $*$. We will use $V = [K]$ to denote the set of transient nodes, and use $\wt{V} := [K] \cup \{ * \}$ to denote the transient nodes together with the absorbing node. On this node set $\wt{V}$, graph $G_t$ encodes transition probabilities and edge lengths: $G_t := \( \{ m_{ij} \}_{i \in V, j \in \wt{V} }, \{ l_{ij}^{(t)} \}_{i \in V, j \in \wt{V}  } \)$, where $m_{ij}$ is the probability of transiting from $i$ to $j$ and $l_{ij}^{(t)} \in [0,1]$ is the length from $i$ to $j$ (at epoch $t$). We gather the transition probabilities among transient nodes to form a transition matrix $M = [m_{ij}]_{i,j \in [K]}$. We make the following assumption about $M$. 

\begin{assumption}
	\label{assumption:transition}
	The transition matrix $M = [m_{ij}]_{i,j \in [K]}$ among transient nodes is primitive.~\footnote{A matrix $M$ is primitive if there exists a positive integer $k$, such that all entries in $ M^k $ is positive.}
	In addition, there is an absolute constant $\rho$, such that $ \| M \|_\infty \le \rho < 1$, where $ \| M \|_{\infty} = \max_{i \in [K]} \sum_{j \in [K]} \rvert m_{ij} \rvert $ is the maximum absolute row sum. 
\end{assumption}

In Assumption \ref{assumption:transition}, the primitivity assumption ensures that we can get to any transient node $v$ from any other node state $u$. The infinite norm of $M$ being strictly less than $1$ means that the random walk will transit to the absorbing node starting from any node (eventually with probability 1). This describes the absorptiveness of the environment. This infinite norm assumption can be replaced by other notions of matrix norms. Unless otherwise stated, we assume that $\rho$ is an absolute constant independent of $K$ and $T$. 

Playing node $j$ at epoch $t$ generates a random walk trajectory $\P_{t,j} := \big( X_{t,0}^{(j)},$ $L_{t,1}^{(j)},$ $X_{t,1}^{(j)},$ $L_{t,2}^{(j)},$ $X_{t,2}^{(j)},$ $\dots,$ $L_{t,H_{t,j}}^{(j)},$ $X_{t,H_{t,j}}^{(j)} \big) $, where 
$X_{t,0}^{(j)} = j$ is the starting nodes, $X_{t, H_{t,j}}^{(j)} = *$ is the absorbing node, $ X_{t,i}^{(j)} $ is the $i$-th node in the random walk trajectory, $L_{t,i}^{(j)}$ is the edge length from $X_{t,i-1}^{(j)}$ to $X_{t,i}^{(j)}$, and $H_{t,j}$ is the number of edges in trajectory $ \P_{t,j} $. For simplicity, we write $ X_{t,i}^{(j)} $ (resp. $L_{t,i}^{(j)}$) as $ X_{t,i} $ (resp. $L_{t,i}$) when it is clear from context.

For the random trajectory $\P_{t,j} := \big( X_{t,0},$ $L_{t,1}, X_{t,1},$ $L_{t,2}, X_{t,2}, \dots, L_{t,H_{t,j}}, X_{t,H_{t,j}} \big) $, the length of the trajectory (or {{hitting time}} of node $ j $ at epoch $t$) is defined as 
    $\L \( \P_{t,j} \) := \sum_{ i = 1 }^{H_{t,j}} L_{t,i}. $
Here we use the edge length to represent the reward of the trajectory. 
In practice, the edge lengths may have real-world meanings. For example, the out-going edge from a node may represent utility (e.g., profit) of visiting this node. 
At epoch $t$, the agent selects a node $J_t \in [K]$ to initiate a random walk, and observe trajectory $\P_{t,J_t}$. 
In stochastic environments, the environment does not change across epochs. Thus for any fixed node $v \in [K]$, the random trajectories $\P_{1,v}, \P_{2,v}, \P_{3,v}, \dots$ are independently identically distributed. 



\section{Information Theoretical Properties} 
\label{sec:info}

Consider the case where the graphs $G_t$ do not change across epochs. To solve this problem, one can estimate the expected hitting times $ \mu_j := \E \left[ \L (\P_{t,j}) \right]$ for all nodes $ j \in [K] $ (and maintain a confidence interval of the estimations). As one can expect, the random walk trajectory reveals more information than a sample of reward. Naturally, this allows us to reduce this problem to a standard (stochastic) MAB problem. 

\subsection{Reduction to Standard MAB} 
\label{sec:reduction}
Recall $ \P_{t,J_t} = \( X_{t,0}, L_{t,1}, X_{t,1}, \dots, L_{t,H_{t,J_t}}, X_{t,H_{t,J_t}} \) $ is the trajectory at epoch $t$. 
For a node $v$ and the trajectories $ \P_{1, J_1}, \P_{2,J_2}, \P_{3,J_3}, \dots $, let $k_{v,i}$ be the index (epoch) of the $i$-th trajectory that covers node $v$. Let $Y_{v, k_{v,i}}$ be the sum of edge lengths between the first occurrence of $v$ and the absorbing node $*$ in trajectory $k_{v,i}$. One has the following proposition due to Markov property. 

\begin{proposition} 
\label{prop:est} 
In the stochastic setting, for any nodes $ v \in [K] $, 
we have, for $\forall t,i \in \mathbb{N}_+, \forall r \in \mathbb{R}$
    \begin{align}  
        \Pr \( Y_{v,k_{v,i}} = r \) = \Pr \( \L \( \P_{t,v} \) = r \) .  
    \end{align} 
\end{proposition} 

\begin{proof} 
    In a trajectory $ \P_{t, J_t} = \( X_{t,0}, L_{t,1}, X_{t,1}, \dots, L_{t,H_{t,J_t}}, X_{t,H_{t,J_t}} \) $, conditioning on $X_{t,i} = j$ being known (and no future information is revealed), the randomness generated by $  L_{t,i+1}, X_{t,i+1}, L_{t,i+2}, X_{t,i+2}, \dots $ is identical to the randomness generated by $ L_{t,1}, X_{t,1}, L_{t,2}, X_{t,2}, \dots $ conditioning on $  X_{t,0} = j $ being fixed. 
    Note that even if each trajectory can visit a node multiple times, only one hitting time sample can be used. This is because extracting multiple sample would break Markovianity, by revealing that the random walk will visit a same node again. 
\end{proof} 
For a node $v \in [K]$, we define 
	$N_t (v) := 1  \vee  \sum_{s < t} \mathbb{I}_{ [ J_s = v ]  } $, and $  
	{N}_t^+ (v)  := 1 \vee \sum_{s < t } \mathbb{I}_{ \[ v \in \P_{s, J_{s} } \] } .$
where $a \vee b = \max \{ a , b \}$. 
In words, $ {N}_t (v) $ is the number of times node $v$ is played, and ${N}_t^+ (v)$ is the number of times node $v$ is covered by a trajectory. 
By Proposition \ref{prop:est}, the information about the node rewards (hitting time to absorbing node) accumulates faster than standard MAB problems. Formally, it holds that $ N_t^+ (v) \ge N_t (v) $. Thus solving this problem is not hard: one can extract the hitting time estimates and apply a standard algorithm (e.g., UCB) based on the estimates. However, {some information is lost when we extract hitting time samples}, since trajectories also carry additional information (e.g, about graph transition and graph structure) but we only extract hitting time samples. Thus the intriguing question to ask is: 
\begin{itemize}
    \item Do we give up too much information by only extracting hitting time samples from trajectories?  \textbf{(Q)} 
\end{itemize}
%
We show that, although more information in addition to reward sample are available from the random walk trajectories, no algorithm can beat an $\Omega \( \sqrt{T} \)$ lower bound in a worst case sense. This provides an answer to \textbf{(Q)}. 
%
\newcount\mycount

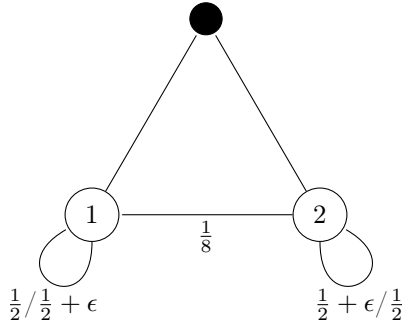
\begin{figure} 
    \centering 
    \begin{tikzpicture}[ 
    shorten > = 1pt, node distance = 3cm and 4cm, 
    el/.style = {inner sep=2pt, align=left, sloped}, 
    every label/.append style = {font=\tiny} ] 
    
    \tikzset{every loop/.style={}} 
    
    \node[draw,circle,inner sep=0.15cm] (1) at (-60:3cm) {2}; 
    \node[draw,circle,inner sep=0.15cm] (2) at (-120:3cm) {1}; 
    
    \node[draw,circle,inner sep=0.15cm, fill=black] (tar) at (0:0) {}; 
    
    \path (1) edge (tar) {}; 
    \path (2) edge (tar) {}; 
    
    \path[-] (1)  edge []  node[el,below]  {$\frac{1}{8} $} (2) ;
    \path[-] (1)  edge [in=-30,out=-90,loop]  node[el,below,rotate=-30]  {$\frac{1}{2} + \epsilon $/$\frac{1}{2}$} (1) ; 
    \path[-] (2) edge [in=-150,out=-90,loop] node[el,below,rotate=30] {$\frac{1}{2}$/$\frac{1}{2} + \epsilon $} (2); 
    \end{tikzpicture} 
    \caption{ 
    Problem instances constructed for Theorem \ref{thm:lower}. The edge labels denote edge transition probabilities in $\J$/$\J'$. The top dark node denotes the absorbing node $*$. Note that node 1 is connected to node 2 with a constant probability. 
    \label{fig:network-lower-bound}}  
\end{figure}

\begin{theorem}\label{thm:lower}
	For any given $T$ and any policy $\pi$, there exists a problem instance $\J$ satisfying Assumption \ref{assumption:transition} such that: (1) The probability of visiting any node from different node is larger than an absolute constant; (2) For any $ \epsilon \in \( 0, \frac{1}{4} \)$, the $T$ step regret of $\pi$ on instance $\J$ is lower bounded by 
	   $ \( \frac{32}{15} \epsilon + O \( \epsilon^2  \) \) T \exp \( - \frac{112}{9} T \( \epsilon^{2} + O \( \epsilon^3 \) \) \) . $ 
	In particular, setting $\epsilon = \frac{1}{4} T^{-1/2} $ gives that there exists a problem instance such that the regret of any algorithm on this instance is lower bounded by $ \Omega \( \sqrt{T} \) $. 
\end{theorem} 
%
%
To prove Theorem \ref{thm:lower}, we construct problem instances such that the optimal nodes are almost indistinguishable. In the construction, we ensure that the nodes visits each other with a constant probability. If instead, the nodes visit each other with an arbitrarily small probability, the nodes are basically disconnected and the problem is too similar to a standard MAB problem. We use the instances illustrated in Fig. \ref{fig:network-lower-bound} for the proof. As shown in Fig. \ref{fig:network-lower-bound}, node 1 and node 2 are connected with probability $\frac{1}{8}$. This construction ensures that the two nodes visits each other with a constant probability. This prevents our construction from collapsing to a standard MAB problem, where the non-absorbing nodes are disconnected. 

\begin{proof}[Proof of Theorem \ref{thm:lower}.] 
We construct two ``symmetric'' problem instances $\J$ and $\J'$ both on two transient nodes $\{ 1,2 \}$ and one absorbing node $*$. All edges in both instances are of length 1. We use ${M} = [ m_{ij}]$ (resp. ${M}' = [m_{ij}'] $) to denote the transition probabilities among transient nodes in instance $\J$ (resp. $\J'$). We construct instances $\J$ and $\J'$ so that $ M = \begin{bmatrix} \frac{1}{2} & \frac{1}{8} \\ \frac{1}{8} &  \frac{1}{2} + \epsilon \end{bmatrix} $ and $M' = \begin{bmatrix} \frac{1}{2} + \epsilon & \frac{1}{8} \\ \frac{1}{8} &  \frac{1}{2} \end{bmatrix}  $, as shown in Figure \ref{fig:network-lower-bound}. 
%
We use $Z_v$ (resp. $Z_v'$) to denote the random variable $ \L \( \P_{t,v} \) $ in problem instance $\J$ (resp. $\J'$). Note that the (expected) hitting time of node 1 in $\J$ can be expressed as $ \E [ Z_1 ] = m_{11} (\E [ Z_1 ] + 1) + m_{21} (\E [Z_2] + 1 ) + (1 - m_{11} - m_{21}) = m_{11} \E [Z_1] + m_{21} \E [Z_2] + 1 $. 
Thus in matrix form we have 
\begin{align*} 
	\[ \E \[ Z_1 \] , \E \[ Z_2 \] \]
	=& \; 
	\[ \E \[ Z_1 \] , \E \[ Z_2 \] \] M + \mathbf{1}^\top , \\ 
	\[ \E \[ Z_1' \] , \E \[ Z_2' \] \]
	=& \; 
	\[ \E \[ Z_1' \] , \E \[ Z_2' \] \] M' + \mathbf{1}^\top , 
\end{align*}  
where $\mathbf{1}$ is the all-one vector. 
Solving the above equations gives, for both instances $\J$ and $\J'$, the optimality gap $\Delta$ is 
\begin{align} 
	\Delta := \left\rvert\E \[ Z_1 \] - \E \[ Z_2 \] \right\rvert = \frac{ 64\epsilon }{ 15 } + O \( \epsilon^2 \) .  \label{eq:gap-lower-bound} 
\end{align} 

Let $\pi$ be any fixed algorithm and let $T$ be any fixed time horizon, we use $ \Pr_{\J, \pi} $ (resp. $ \Pr_{\J', \pi} $) to denote the probability measure of running $\pi$ on instance $\J$ (resp. $ \J' $) for $T$ epochs. 

Since the event $\{ J_t = 1 \}$ ($t \le T$) is measurable by both $\Pr_{\J, \pi}$ and $\Pr_{\J', \pi}$, we have 
\begin{align}
	\Pr_{\J, \pi} ( J_t \neq 1 ) + \Pr_{\J', \pi} (J_t = 1) 
	=& \; 
	1 - \Pr_{\J, \pi} ( J_t = 1 ) + \Pr_{\J', \pi} (J_t = 1) \nonumber \\
	\overset{(i)}{\ge}& \; 
	1 - \| \Pr_{\J, \pi} - \Pr_{\J', \pi} \|_{TV} \nonumber \\
	\overset{(ii)}{\ge}& \; 
	1 - \sqrt{ 1 - \exp \( D_{kl} ( \Pr_{\J, \pi} \| \Pr_{\J', \pi} ) \) } \nonumber \\ 
	\overset{(iii)}{\ge}& \;  
    \frac{1}{2} \exp \( - D_{kl} ( \Pr_{\J, \pi} \| \Pr_{\J', \pi} ) \) \label{eq:event-to-KL} ,  
\end{align}
where $(i)$ uses the definition of total variation, $(ii)$ uses the Bretagnolle-Huber inequality, and $(iii)$ uses that $ 1 - \sqrt{1-x} \ge \frac{1}{2}x $ for all $x \ge 0$. 

Let $\Q_i$ (resp. $\Q_i'$) be the probability measure generated by playing node $ i $ in instance $\J$ (resp. $\J'$). We can then decompose $\Pr_{\J, \pi}$ by
\begin{align}
    \Pr_{\J, \pi} &= \Q_{J_1} \Pr \( J_1 \rvert \pi \) \Q_{J_2} \Pr ( J_2 \rvert \pi , J_1) \cdots \Q_{J_T} \Pr ( J_T \rvert \pi, J_1, J_2, \dots, J_{t-1} ), \nonumber \\
    \Pr_{\J', \pi} &= \Q_{J_1}' \Pr \( J_1 \rvert \pi \) \Q_{J_2}' \Pr ( J_2 \rvert \pi , J_1)  \cdots \Q_{J_T}' \Pr ( J_T \rvert \pi, J_1, J_2, \dots, J_{t-1} ). \nonumber
\end{align}

By chain rule for KL-divergence, we have 
\begin{align}
    &D_{kl} ( \Pr_{\J, \pi} \| \Pr_{\J', \pi} ) \nonumber \\
    =& 
    \sum_{J_1 \in \{ 1,2 \} } \Pr(J_1\rvert \pi ) D_{kl} ( \Q_{J_1} \| \Q_{J_1}' ) \nonumber \\
    &+ \sum_{t=2}^T \sum_{J_t \in \{ 1,2 \} } \Pr(J_t\rvert \pi , J_{1}, \cdots, J_{t-1}) D_{kl} ( \Q_{J_t} \| \Q_{J_t}' ) . \label{eq:simple-chain-0} 
\end{align} 

Since the policy must pick one of node 1 and node 2, from (\ref{eq:simple-chain-0}) we have 
\begin{align} 
	D_{kl} ( \Pr_{\J, \pi} \| \Pr_{\J', \pi} ) \le \sum_{t=1}^T \sum_{ i=1 }^2  D_{kl} ( \Q_i \| \Q_i' ),  \label{eq:simple-chain} 
\end{align} 
which allows us to remove dependence on policy $\pi$. 

Next we study the distributions $\Q_i$. With edge lengths fixed, the sample space of this distribution is $ \cup_{ h=1}^\infty \{ 1,2\}^h $, since length of the trajectory can be arbitrarily long, and each node on the trajectory can be either of $\{ 1, 2 \} $. To describe the distribution $\Q_i$ and  $ \Q_i' $, we use random variables $X_0, X_1, X_2, X_3, \dots \in \{ 1, 2, * \}$ (with $X_0 = i$), where $X_k$ is the $k$-th node in the trajectory generated by playing $ i $. 

By Markov property we have, for $i,j \in \{ 1, 2 \}$, 
\begin{align} 
    \Q_i \( X_{k+1}, X_{k+2}, \dots \rvert X_k = j \) &= \Q_j \quad \label{eq:6} \text{and}  \\
    \Q_i' \( X_{k+1}, X_{k+2}, \dots \rvert X_k = j \) &= \Q_j', \quad \forall k \in \mathbb{N}_+. \nonumber 
\end{align} 

Since we can decompose $ \Q_i $ by $ \Q_i = \Q_i (X_1) \Q_i ( X_2, X_3, \dots, \rvert X_1 ) $. 
Thus by chain rule of KL-divergence, we have 
\begin{align} 
    &\; D_{kl} (\Q_1 \| \Q_1' ) \nonumber \\
    =&\; D_{kl} ( \Q_1 (X_1) \| \Q_1' (X_1) ) \nonumber \\
    &+ \Q_1 (X_1 \hspace*{-2pt} = \hspace*{-2pt} 1) D_{kl} ( \Q_1 (X_2, \cdots \rvert X_1 \hspace*{-2pt} = \hspace*{-2pt} 1) \| \Q_1' (X_2, \cdots \rvert X_1 \hspace*{-2pt} = \hspace*{-2pt} 1) )  \nonumber \\
    &+ \Q_1 (X_1 \hspace*{-2pt} = \hspace*{-2pt} 2) D_{kl} ( \Q_1 (X_2, \cdots \rvert X_1 \hspace*{-2pt} = \hspace*{-2pt} 2) \| \Q_1' (X_2, \cdots \rvert X_1 \hspace*{-2pt} = \hspace*{-2pt} 2) ) \nonumber \\
    =&\; D_{kl} ( \Q_1 (X_1) \| \Q_1' (X_1) ) \nonumber \\
    &+ \Q_1 (X_1 = 1)  D_{kl} ( \Q_1 \| \Q_1' ) + \Q_1 (X_1 = 2)  D_{kl} ( \Q_{2} \| \Q_{2}' ) , \label{eq:for-KL}
\end{align} 
where the last line uses (\ref{eq:6}). A similar argument gives
\begin{align}
    & \; D_{kl} (\Q_2 \| \Q_2' ) \nonumber \\
    =& \; D_{kl} ( \Q_2 (X_1) \| \Q_2' (X_1) ) \nonumber \\
    &+ \Q_2 (X_1 = 1)  D_{kl} ( \Q_1 \| \Q_1' ) + \Q_2 (X_1 = 2)  D_{kl} ( \Q_{2} \| \Q_{2}' ) , \label{eq:for-KL2}
\end{align} 

Since $ D_{kl} ( \Q_i (X_1) \| \Q_i' (X_1) )  
= \frac{7}{3} \epsilon^2 + O \( \epsilon^3 \)   $ for $i \in \{ 1 , 2 \}$, the above (Eq. \ref{eq:for-KL} and Eq. \ref{eq:for-KL2}) gives 
\begin{align} 
    D_{kl} ( \Q_i \| \Q_i' ) &= \frac{56 \epsilon^2}{9} + O \( \epsilon^3 \) . \label{eq:KL} 
\end{align} 

Combining (\ref{eq:simple-chain}) and (\ref{eq:KL}) gives
\begin{align} 
	D_{kl} \( \Pr_{\J, \pi} \| \Pr_{\J', \pi} \) \le \frac{112}{9} T \( \epsilon^{2} + O \( \epsilon^3 \)  \) . \label{eq:prob-gap} 
\end{align} 

Let $\Reg (T)$ (resp. $\Reg' (T)$) be the $T$ epoch regret in instance $\J$ (resp. $\J'$). Recall, by our construction, node 1 is suboptimal in instance $\J$ and node 1 is optimal in instance $\J'$. Since the optimality gaps in $\J$ and $\J'$ are the same (Eq. \ref{eq:gap-lower-bound}), we have, 
\begin{align}
	&\; \Reg (T) + \Reg' (T) \nonumber \\
	\ge&\; \Delta  \sum_{t=1}^T \Big( \Pr_{\J, \pi} \( J_t \neq 1 \)  +  \Pr_{\J', \pi} \( J_t = 1 \) \Big) \nonumber \\
	\ge&\; 
	\frac{1}{2} \Delta T  \exp \( - D_{kl} ( \Pr_{\J, \pi} \| \Pr_{\J', \pi} ) \) \tag{by Eq. \ref{eq:event-to-KL}} \\ 
	\ge&\;
	\frac{1}{2} \Delta T \exp \( - \frac{112}{9} T \( \epsilon^{2} + O \( \epsilon^3 \)  \) \) \tag{by Eq. \ref{eq:prob-gap}} \\
	\ge&\;
	\( \frac{32}{15} \epsilon + O \( \epsilon^2  \) \) T 
	\exp \( - \frac{112}{9} T \( \epsilon^{2} + O \( \epsilon^3 \) \) \) .
	\tag{by Eq. \ref{eq:gap-lower-bound}} 
\end{align} 
\end{proof}


In Theorem \ref{thm:lower}, the two-nodes case has been covered. A similar result for the $ K$-node case is in Theorem \ref{thm:K-lower}. 

\begin{theorem} 
    \label{thm:K-lower} 
    Let Assumption \ref{assumption:transition} be true. 
    Given any number of epochs $T$ and policy $ \pi $, there exists $K$ problem instances $\J_1, \cdots, \J_K$, such that (1) all problems instances have $K$ nodes, (2) all transient node are connected with the same probability and the probability of hitting the absorbing node from any transient node is a constant independent of $T$ and $K$, and (3) for any policy $\pi$, 
    \begin{align*} 
        \max_{k \in [K]} \E_{\J_k, \pi} \[ \Reg (T) \] \ge \frac{1}{8\sqrt{2}} \sqrt{KT}, 
    \end{align*} 
    where $ \E_{\J_k, \pi} $ is the expectation with respect to the distribution generated by instance $ \J_k $ and policy $\pi$. 
\end{theorem} 

The proof of Theorem \ref{thm:K-lower} follows a recipe similar to that of Theorem \ref{thm:lower}, and the details are deferred to the Appendix. 


\section{Algorithms for Multi-Armed Bandits with Random Walk Feedback}
\label{sec:algs}

Perhaps the two most well-known algorithms for bandit problems are the UCB algorithm and the EXP3 algorithm. In this section, we study the behavior of these two algorithms on bandit problems with random walk feedback. 

\subsection{UCB Algorithm for the Stochastic Setting} 

As discussed previously, the problem with random walk feedback can be reduced to standard MAB problems, and the UCB algorithm solves this problem as one would expect. We present here the UCB algorithm and provide regret analysis for it. Recall the regret is defined as 
\begin{align} 
	\Reg(T) = \max_{i \in [K]} \sum_{t=1}^T \mu_i -  \sum_{t=1}^T \mu_{{}_{J_t}}, 
\end{align} 
where $J_t$ is the node played by the algorithm (at epoch $t$), and $\mu_{{}_{j}} = \E \[ \L \( \P_{t,j} \) \] $ is the expected hitting time of node $ j $. 

For a transient node $v \in [K]$ and $ n $ trajectories (at epochs $ k_{v,1}, k_{v,2}, \dots, k_{v,n} $ ) that cover node $v$, the hitting time estimator of $v$ is computed as
    $\wt{Z}_{v,n} := \frac{1}{n} \sum_{i=1}^n Y_{v, k_{v,i} }. $
Since $v \in \P_{t, v}$, $ Y_{v,k_{v,i}} $ is an identical copy of the hitting time $\L \( \P_{t,v} \) $ (Proposition \ref{prop:est}). 

We also need confidence intervals for our estimators. 
Given $ N_t^+  (v) $ trajectories covering $v$, the confidence terms (at epoch $t$) are
	$\wt{C}_{N_t^+  (v) , t} \hspace{-2pt}:= \wtC{ N_t^+  (v)  }, $ 
where $\xi_t = \xit$. Here $\xi_t$ serves as a truncation parameter since the reward distribution is not sub-Gaussian. Alternatively, one can use robust estimators for bandits with heavy-tail reward for this task \cite{bubeck2013bandits}. 

At each time $t$, we play a node $J_t$ so that 
\begin{align*}
    J_t \in \arg\max_v \wt{Z}_{ v, N_t^+(v) } + \wt{C}_{ N_t^+(v), t} . 
\end{align*}
This strategy is described in Algorithm \ref{alg:opt}.
\begin{algorithm}[h!]
    \caption{}  
    \label{alg:opt}
    \begin{algorithmic}[1] 
        \State \textbf{Input:} A set of nodes $[K]$. Parameters: a constant $ \rho $ that bounds the spectral radius of $M$. 
        \State \textbf{Warm up:} Play each node once to initialize. Observe trajectories. 
        \State 
        For any $v \in [K]$, define the decision index  
        			$I_{v,N_t^+(v),t} = \wt{Z}_{ v, N_t^+(v) } + \wt{C}_{ N_t^+(v) ,t} .$ 
        \FOR {$t = 1, 2, 3,  \dots $} 
        	\State Select $J_t$, such that 
        			$J_t \in \arg\max_{ v \in V } I_{v, N_t^+ (v) , t },$
        		with ties broken arbitrarily. 
        	\State 
        	Observe the trajectory $ \P_{t, v_t} := \big\{ X_{t,0}, L_{t,1}, X_{t,1}, L_{t,2}, X_{t,2}, \cdots, L_{t,H_{t,v_t}} X_{t,H_{t,J_t}} \big\} $. Update ${N}_t^+ (\cdot)$ and UCBs for all $v \in [K]$. 
        \ENDFOR 
    \end{algorithmic} 
\end{algorithm} 

Similar to the UCB algorithm for standard MAB problems, Algorithm \ref{alg:opt} obtains $\widetilde{\mathcal{O}} ( \sqrt{T} )$ mini-max regret (gap-independent regret) and $\widetilde{\mathcal{O}} (1)$ asymptotic regret (gap-dependent regret). Such results can be derived from the observations discussed in Section \ref{sec:reduction}. 

\subsection{EXP3 Algorithm for Adversarially Chosen Edge Lengths}

In this section, we consider the case in which the network structure $G_t$ changes over time, and study a version of this problem in which the adversary alters edge length across epochs: In each epoch, the adversary can arbitrarily pick edge lengths $l_{ij}^{(t)}$ from $[0,1]$. In this case, the performance is measured by the regret against playing any fixed node $i \in [K]$: 
    $\Regadv_i(T) = \sum_{t=1}^T l_{t,i} -  \sum_{t=1}^T l_{t,J_t} ,$  
where $J_t$ is the node played in epoch $t$, and $l_{t,j} = \E \[  \L \( \P_{t,i} \)  \]$. Since $  \L \( \P_{t,j} \) $ concentrates around $l_{t,j}$, a high probability bound on $\Regadv_i(T)$ naturally provides a high probability bound on $ \sum_{t=1}^T \L \( \P_{t,i} \) -  \sum_{t=1}^T \L \( \P_{t,J_t} \) $. 

We define a notion of centrality that will aid our discussions. 
\begin{definition}
	\label{def:centrality}
	Let $ X_0, X_1, X_2, \dots, X_\tau = * $ be nodes on a random trajectory. 
	Under Assumption \ref{assumption:transition}, we define, for node $v \in [K]$,  
	$\alpha_{v} := \min_{u \in [K], \; u \neq v} \Pr \left( v \in \{ X_1, X_2, \dots, X_\tau \} \rvert X_0 = u \right)$ 
	to be the \textbf{hitting centrality} of node $v$. We also define $\alpha = \min_v \alpha_v$
\end{definition} 
For a node with positive the hitting centrality, information about it is revealed even if it is not played. This quantity will show up in the regret bound in some cases, as we will discuss later. 

We will use a version of the exponential weight algorithm to solve this adversarial problem. Also, a high probability guarantee is provided using a new concentration lemma (Lemma \ref{lem:bound}). As background, exponential weights algorithms maintain a probability distribution over the choices. This probability distribution gives higher weights to historically more rewarding nodes. In each epoch, a node is sampled from this probability distribution, and information is recorded down. 
To formally describe the strategy, we define some notations now. 
We first extract a sample of $ \L \( \P_{t,j} \) $ from the trajectory $\P_{t,J_{t}}$, where $J_{t}$ is the node played in epoch $t$. 

Given $ \P_{t, J_t} = \( X_{t,0}, L_{t,1}, X_{t,1}, \dots, L_{t,H_{t,J_t}}, X_{t,H_{t,J_t}} \) $, we define, for $v \in [K]$, 
\begin{align}
    Y_v ( \P_{t,J_t} ) = \max_{i: 0\le i < H_{t,J_t} }  \; \I_{ \[ X_{t,i} = v \] } \cdot \L_{i} (\P_{t,J_t} ) , \label{eq:est-extract}
\end{align}
where $ \L_{i} ( \P_{t,J_t} ) := \sum_{ k=i+1 }^{H_{t,J_t}} L_{t,k} $. In words, $\L_{i} ( \P_{t,J_t} )$ is the distance (sum of edge lengths) from the first occurrence of $v$ to the absorbing node. 

By the principle of Proposition \ref{prop:est}, if node $ i$ is covered by trajectory $\P_{t, J_t}$, $ Y_v (\P_{t,J_t}) $ is a sample of $ \L \( \P_{t,i} \) $. We define, for the trajectory $ \P_{t,J_t} = \{ X_{t,0}, L_{t,1}, X_{t,1}, L_{t,2} ,\dots, L_{t,H_{t,J_t }}, X_{t, H_{t,J_t }} \} $, 
    $Z_{t,v} := Y_v ( \P_{t,J_t} ), \quad \forall v \in [K], $ 
where $Y_v ( \P_{t,J_t} )$ is defined above. 

Define $\I_{t,ij} := \I_{ \[i\in \P_{t,J_t} \text{ and } j\in \P_{t,J_t} , \, Y_i ( \P_{t,J_t} ) > Y_j ( \P_{t,J_t} ) \] }$. This indicator random variable is $1$ iff $ i $ and $j$ both show up in $\P_{t,J_t}$ and the first occurrence of $ j $ is after the first occurence of $i$. We then define 
    $\wh{q}_{t,ij} :=\frac{ \sum_{s = 1}^{t-1} \I_{s,ij}  }{N_t^+ (i) } ,  $
which is an estimator of how likely $j$ is visited via a trajectory starting at $i$. In other words, $\wh{q}_{t,ij}$ is an estimator of $q_{ij} := \Pr \( j \in \P_{t,i} \)$, which is the probability of $j$ being visited by a trajectory from $i$. 
Using $\wh{q}_{t,ij}$ and sample $ {Z}_{t,i} $, we define an estimator for $ \frac{ l_{t,j} - B }{B} $ as
\begin{align} 
    \wh{Z}_{t,i} &:= \frac{ \frac{ {Z}_{t,i} - B }{B} \I_{ \[ i \in \P_{t,J_{t}} \] } + \beta }{ p_{ti} + \sum_{j \neq i} \wh{q}_{t,ji} p_{tj}  } , \quad \forall i \in [K],  \label{eq:est-adv} 
\end{align} 
where $B, \beta$ are algorithm parameters ($\beta \le \alpha$ and $B$ to be specified later). Here $\beta$ serves as a parameter for implicit exploration \cite{neu2015explore}. We use estimate for $ \frac{ l_{t,j} - B }{B} $ instead of an estimate for $l_{t,j}$. This shift can guarantee that the estimator $\wh{Z}_{t,i}$ is between $[-1,0]$ with high probability. Such shifting in estimator has been used as a common trick for variance reduction for the EXP3 algorithms \citep[e.g.,][]{lattimore_szepesvari_2020}. In addition, a small bias is introduced via $\beta$. 
With the estimators $\wh{Z}_{t,i} $, we define 
$
    \wh{S}_{t,j} = 
        \sum_{ s = 1 }^{t-1} \wh{Z}_{s,i} .$ 
By convention, we set $\wh{S}_{0,i} = 0 $ for all $i \in [K]$. The probability of playing $i$ in epoch $t$ is defined as
\begin{align}
    p_{ti} := 
    \begin{cases} 
        \frac{1}{K}, \quad \text{ if } t = 1, \\
        \frac{ \exp \( \eta \wh{S}_{t-1, i} \) }{ \sum_{ j=1 }^K \exp \( \eta \wh{S}_{t-1, j} \) } , \quad \text{ if } t \ge 2, 
    \end{cases} \label{eq:def-ptj} 
\end{align}  
where $\eta$ is the learning rate. 

Against any arm $j \in [K]$, following the sampling rule (\ref{eq:def-ptj}) can guarantee an $ \wt{\mathcal{O}} \( \sqrt{T} \) $ regret bound. 
We now summarize our strategy in Algorithm \ref{alg:adv}, and state the performance guarantee in Theorem \ref{thm:adv}.

\begin{algorithm}[h!] 
    \caption{}  
    \label{alg:adv} 
    \begin{algorithmic}[1] 
        \State \textbf{Input:} A set of nodes $ [K]$, transition matrix $M$, total number of epochs $T$, probability parameter $\epsilon \in \(0, \frac{1}{(1 - \rho ) KT } \)$. Algorithm parameters: $B = \advB$, $\eta = \frac{1}{\sqrt{\kappa T}}$, $\beta = \frac{1}{\sqrt{\kappa T}}$.  
        \FOR {$t = 1, 2, 3,  \dots , T$} 
        	\State Randomly play $J_t \in [K]$, such that 
        			$\Pr \( J_t = i \) = p_{ti}, \forall i \in [K], $
        		where $p_{ti}$ is defined in (\ref{eq:def-ptj}). 
        	\State Observe the trajectory $ \P_{t,J_t} $. 
        	Update estimates $ \wh{Z}_{t,j} $ according to (\ref{eq:est-adv}). 
        \ENDFOR
    \end{algorithmic} 
\end{algorithm}


A high probability performance guarantee of Algorithm \ref{alg:adv} is below in Theorem \ref{thm:adv}.

\begin{theorem} \label{thm:adv}
    Let $ \kappa:= 1 + \sum_j \frac{1 - \sqrt{\alpha_j }}{ 1 + \sqrt{\alpha_j } } $, where $\alpha_j$ is the hitting centrality of $j$. Fix any $i \in [K]$. If the time horizon $T$ and algorithm parameters satisfies $\epsilon \le \frac{1}{T}$, $ B = \advB $, then with probability exceeding $1 - \Ot (\epsilon)$, 
    \begin{align*}
        \Regadv_i (T)  
        \le& \; 
        B \frac{\log (T / \epsilon^2 ) }{\beta} + B \kappa \beta T + \wt{\mathcal{O}} \( B \sqrt{\beta \kappa T } \) \\
        &+ 
        \frac{ B \log K }{\eta } + B \eta \kappa T + \wt{\mathcal{O}} \( B \eta \sqrt{ (1 + \beta \kappa) T } \) . 
    \end{align*}
\end{theorem}

If we set $ \eta = \beta = \Theta \( \frac{1}{\sqrt{ \kappa T} } \) $, Algorithm \ref{alg:adv} achieves $\Regadv_i (T) \le \wt{\mathcal{O}} \( \sqrt{\kappa T } \) $. In Figure \ref{fig:adv-reg}, we provide a plot of $ f (x) = \frac{1-\sqrt{x}}{1 + \sqrt{x}} $ with $x \in [0,1]$. This figure illustrates how the regret scales with $\kappa$, and shows that a multiplicative factor is saved. An empirical comparison between Algorithm \ref{alg:adv} and the standard EXP3 algorithm (without using the estimator in Eq. \ref{eq:est-extract}) is shown in Section \ref{sec:exp}. 


\begin{figure}
    \centering
    \includegraphics[width=5.5cm]{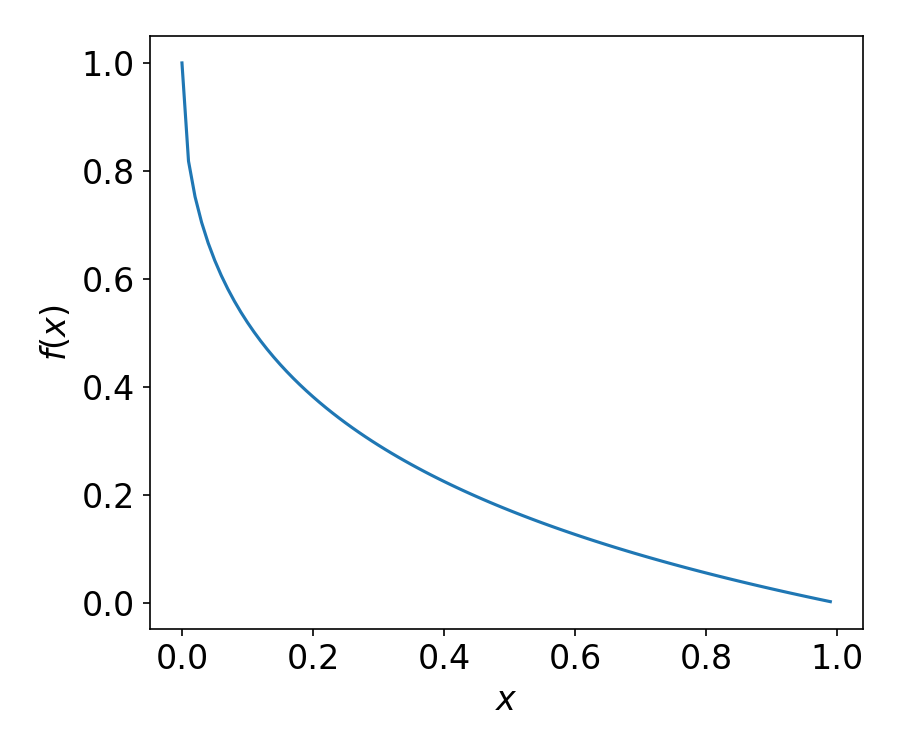}
    \caption{A plot of function $f (x) = \frac{1 - \sqrt{x}}{1 + \sqrt{x}}$, $x \in [0,1]$. \label{fig:adv-reg} } 
\end{figure}

\subsubsection{Analysis of Algorithm \ref{alg:adv}}

In this section we present a proof for Theorem \ref{thm:adv}. In general, the proof follows the general recipe for the analysis of EXP3 algorithms, which is in turn a special case of the Follow-The-Regularized-Leader (FTRL) or the Mirror Descent framework. 
Below we include some non-standard intermediate steps due to the special feedback structure of the problem studied. More details are defered to the Appendix.

By the exponential weights argument 
	\citep{littlestone1994weighted, auer2002nonstochastic}, it holds that,
	under event $\mathcal{E}_T (B) : =  \{Z_{t,j} \le B \text{ for all } t = 1, 2, \cdots , T, \text{ and } j \in [ K ] \}$, 
    \begin{align}
        \sum_{t=1}^T \wh{Z}_{t,i} - \sum_{t=1}^T \sum_{j} p_{tj} \wh{Z}_{t,j} 
        \le& \frac{ \log K}{\eta} + \eta \sum_{t=1}^T \sum_j \frac{ p_{tj} }{ \wt{p}_{tj} ^2 } \I_{ \[ j \in \P_{t, J_t} \] }. \label{eq:classic}
    \end{align} 
    
To link the regret ($\sum_{t=1}^T l_{t,i} -  \sum_{t=1}^T l_{t,J_t}$) to $ \sum_{t=1}^T \wh{Z}_{t,i} - \sum_{t=1}^T \sum_{j} p_{tj} \wh{Z}_{t,j}  $ and to bound $ \sum_{t=1}^T \sum_j \frac{ p_{tj} }{ \wt{p}_{tj} ^2 } \I_{ \[ j \in \P_{t, J_t} \] } $, we use the results in Lemmas \ref{lem:bound} and \ref{lem:simple-con}. 

\begin{lemma} 
	\label{lem:bound} 
    For any 
    $\epsilon \in (0,1)$ and $T \in \mathbb{N}$, such that $\epsilon \le \frac{1}{T}$ and $T \ge 10$, it holds that
    \begin{align*}
        \Pr \( \sum_{t} \frac{ l_{t,i} - B }{B} - \sum_{t} \wh{Z}_{t,i}  \le \frac{\log \( T/ \epsilon^2 \) }{\beta} \) \ge 1 - \Ot \( \epsilon \). 
    \end{align*} 
\end{lemma} 

Lemma \ref{lem:bound} can be viewed as a one-side regularized version of the Hoeffding's inequality, with a regularization parameter $\beta$. Its proof uses the Markov inequality and can be found in the Appendix. 

\begin{lemma} 
    \label{lem:simple-con}
	With probability exceeding $1 - \mathcal{O} ( \epsilon ) $, we have 
	\begin{align*}
		(i)& \quad  \sum_{ t } \sum_j \frac{ p_{tj} }{ \wh{p}_{tj}^2 } \I_{ \[ j \in \P_{t, J_t} \] } \le  \kappa T  + \Ot \(  \sqrt{T \log (1/\epsilon)}  \), \\
		(ii)&\quad \sum_t \sum_j p_{tj} \wh{Z}_{t,j} - \sum_t  \frac{ l_{t,J_t} - B }{B} \le  \kappa \beta T + \O \( \sqrt{ \( 1 + \beta \kappa \) T \log (1/\epsilon)} \) . 
	\end{align*} 
\end{lemma} 

\begin{proof}[Proof of Lemma \ref{lem:simple-con}] 
    
    Firstly, it holds with high probability that  $ \wh{p}_{tj} =  \wt{p}_{tj} + \mathcal{O} \( \frac{1}{t} \) $ (Lemma \ref{lem:concen-q} in the Appendix). Therefore $ \frac{p_{tj}}{ \wh{p}_{tj}  } = \frac{p_{tj}}{ \wt{p}_{tj} + \mathcal{O} \( \frac{1}{t} \)  } = \frac{p_{tj}}{ \wt{p}_{tj} } + \mathcal{O} \( \frac{1}{t} \) $. 
    
    Let $\E_t$ denote the expectation conditioning on all randomness right before the $t$-th epoch.    Since $p_{t,j}$ and $  \wt{p}_{tj}$ are known before the $t$-th epoch, we have $ \E_t \[ \frac{ p_{tj} }{ \wh{p}_{tj}^2 } \I_{ \[ j \in \P_{t, J_t} \] } \] = \frac{ p_{tj} }{ \wh{p}_{tj}^2 } \E_t \[  \I_{ \[ j \in \P_{t, J_t} \] } \] =  \frac{p_{tj}}{ \( \wt{p}_{tj} + \mathcal{O} \( \frac{1}{t} \) \)^2 } \wt{p}_{tj} = \frac{p_{tj}}{ \wt{p}_{tj} } + \mathcal{O} \( \frac{1}{t} \) $. Thus the random variables 
    \begin{align*}
        \left\{ \frac{ p_{tj} }{ \wh{p}_{tj}^2 } \I_{ \[ j \in \P_{t, J_t} \] } - \frac{p_{tj}}{ \wt{p}_{tj} } + \mathcal{O} \( \frac{1}{t} \) \right\}_t
    \end{align*}
    form a martingale difference sequence for any $j$.

    Applying the Azuma's inequality to this martingale difference sequence gives, 
    with porbability exceeding $1 - \mathcal{O} (\epsilon)$, 
    \begin{align} 
        \sum_{ t } \sum_j \frac{ p_{tj} }{ \wh{p}_{tj}^2 } \I_{ \[ j \in \P_{t, J_t} \] } 
        &\le 
        \sum_{ t } \sum_j \frac{ p_{tj} }{ \wt{p}_{tj} } + \Ot \(  \sqrt{T \log (1/\epsilon)}  \) + \sum_{ t } \sum_j \mathcal{O} \( \frac{1}{t} \) \nonumber \\ 
        &\le 
        \sum_{ t } \sum_j \frac{ p_{tj} }{ \wt{p}_{tj} } + \Ot \(  \sqrt{T \log (1/\epsilon)}  \). \label{eq:for-kappa} 
    \end{align} 
    
    
    Since $ \frac{p_{tj}}{ \wt{p}_{tj} } \le p_{tj} + \frac{ 1 - \sqrt{ \alpha_j} }{ 1 + \sqrt{ \alpha_j} } $ (Proposition \ref{prop:quick-bound} in Appendix), we have 
    \begin{align} 
        \sum_{ t } \sum_j \frac{ p_{tj} }{ \wt{p}_{tj} } \le \sum_{ t } \sum_j \( p_{tj} + \frac{ 1 - \sqrt{ \alpha_j} }{ 1 + \sqrt{ \alpha_j} } \) \le \kappa T. \label{eq:for-kappa2}
    \end{align} 
    Combining (\ref{eq:for-kappa}) and (\ref{eq:for-kappa2}) finishes the proof for item $(i)$. 
    
    
    For the second inequality in the lemma statement, we first note that $Z_{t,j} \le B$ with high probability for all $t $ and $j$. Thus we have 
    \begin{align*} 
		\sum_j p_{tj} \wh{Z}_{t,j} 
		=& \; 
		\sum_j p_{tj} \frac{ \frac{ Z_{t,j} - B }{B} \I_{ \[ j \in \P_{t,J_t} \] } + \beta }{ \wh{p}_{tj} } \\ 
		\overset{\textcircled{1}}{=}& \; 
		\sum_j p_{tj} \frac{ \frac{ Z_{t,j} - B }{B} \I_{ \[ j \in \P_{t,J_t} \] } }{ \wt{p}_{tj} } + \sum_j \frac{ \beta p_{tj} }{ \wt{p}_{tj} } + \Ot \( \frac{ 1 }{t} \) \\ 
		\le& \; 
		\beta \kappa + \Ot \( \frac{ 1 }{t} \), 
	\end{align*} 
    where \textcircled{1} uses a Taylor expansion (to replace $ \wh{p}_{tj} $ with $\wt{p}_{tj}$), and the last line uses  $ \frac{p_{tj}}{ \wt{p}_{tj} } \le p_{tj} + \frac{ 1 - \sqrt{ \alpha_j} }{ 1 + \sqrt{ \alpha_j} } $ (Proposition \ref{prop:quick-bound} in the Appendix) and that $ Z_{t,j} $ is smaller than $B$ with high probability. Also, $ l_{t,j} $ is smaller than $B$ for all $t,j$. Since $\E_t \left[ \sum_{j} p_{tj} \frac{ \frac{Z_{t,j} - B}{B} \I_{ \[ j \in \P_{t,J_t} \] } }{ \wt{ p }_{tj} } \right] = \E_t \[ \frac{l_{t,J_t} - B}{B} \] $, we know $ \left\{ \( \sum_j p_{tj} \wh{Z}_{t,j} -  \frac{l_{t,J_t} - B}{B } + \beta \kappa + \Ot \( \frac{ 1 }{t} \) \) \right\}_t $ is a super-martingale difference sequence. 
	We can now apply the Azuma's inequality to this super-martingale sequence and get 
	\begin{align} 
		\sum_t \sum_j p_{tj} \wh{Z}_{t,j} - \sum_t \frac{ l_{t,J_t} - B }{B} \le  \kappa \beta T + \Ot \( \sqrt{  \beta \kappa T   } \) , 
	\end{align} 
	with high probability. 
\end{proof}


\begin{proof}[Proof of Theorem \ref{thm:adv}] 
    By the above results, we have 
    \begin{align} 
        \sum_{t=1}^T l_{t,i} - \sum_{t=1}^T l_{t,J_t}
        =& \;  
        B \( \sum_t \frac{l_{t,i} - B}{B} - \sum_t \frac{l_{t,J_t} - B}{B} \) \nonumber \\
        \le& \;  
        B \( \sum_t \frac{l_{t,i} - B}{B} - \sum_t \wh{Z}_{t,i} \) \nonumber \\
        & + B \( \sum_t \sum_j p_{tj} \wh{Z}_{t,j} - \sum_t \frac{l_{t,J_t} - B}{B} \) \nonumber  \\ 
        & + B \( \wh{Z}_{t,i} - \sum_t \sum_j p_{tj} \wh{Z}_{t,j} \) \nonumber \\
        \overset{(i)}{\le}& \; 
        B \frac{\log (T / \epsilon^2 ) }{\beta} + B \kappa \beta T + \wt{\mathcal{O}} \( B \sqrt{\beta \kappa T } \) \nonumber \\
        &+ 
        \frac{ B \log K }{\eta } + B \eta \sum_t \sum_j \frac{p_{tj}}{\wt{p}_{tj}^2 } \I_{ \[ j \in \P_{t,J_t} \] } \nonumber \\
        \overset{(ii)}{\le}& \; 
        B \frac{\log (T / \epsilon^2 ) }{\beta} + B \kappa \beta T + \wt{\mathcal{O}} \( B \sqrt{\beta \kappa T } \) \nonumber \\
        &+ 
        \frac{ B \log K }{\eta } + B \eta \kappa T + \wt{\mathcal{O}} \( B \eta \sqrt{ (1 + \beta \kappa) T } \) , \nonumber 
    \end{align} 
    where $(i)$ uses Lemma \ref{lem:bound}, Lemma \ref{lem:simple-con}, and (\ref{eq:classic}), and $(ii)$ uses Lemma \ref{lem:simple-con}. 
\end{proof}

\usetikzlibrary[topaths]
\newcount\mycount

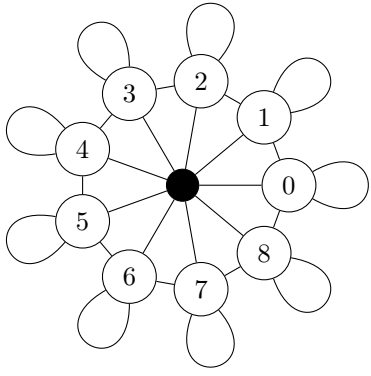
\begin{figure}
    \centering
    \begin{tikzpicture}[scale = 0.7][transform shape]
    
    \node[draw,circle,inner sep=0.15cm] (1) at (0:2cm) {0};
    \node[draw,circle,inner sep=0.15cm] (2) at (40:2cm) {1};
    \node[draw,circle,inner sep=0.15cm] (3) at (80:2cm) {2};
    \node[draw,circle,inner sep=0.15cm] (4) at (120:2cm) {3};
    
    \node[draw,circle,inner sep=0.15cm] (5) at (160:2cm) {4};
    \node[draw,circle,inner sep=0.15cm] (6) at (200:2cm) {5};
    
    \node[draw,circle,inner sep=0.15cm] (7) at (240:2cm) {6};
    \node[draw,circle,inner sep=0.15cm] (8) at (280:2cm) {7};
    
    \node[draw,circle,inner sep=0.15cm] (9) at (320:2cm) {8};
    
    \node[draw,circle,inner sep=0.15cm, fill=black] (tar) at (320:0) {};
    
    \path (1) edge (2) {};
    \path (2) edge (3) {};
    \path (3) edge (4) {};
    \path (4) edge (5) {};
    \path (5) edge (6) {};
    \path (6) edge (7) {};
    \path (7) edge (8) {};
    \path (8) edge (9) {};
    \path (9) edge (1) {};
    
    \path (1) edge (tar) {};
    \path (2) edge (tar) {};
    \path (3) edge (tar) {};
    \path (4) edge (tar) {};
    \path (5) edge (tar) {};
    \path (6) edge (tar) {};
    \path (7) edge (tar) {};
    \path (8) edge (tar) {};
    \path (9) edge (tar) {};
    
    \draw (1) to [out=30,in=-30,looseness=8] (1) {};
    \draw (2) to [out=70,in=10,looseness=8] (2) {};
    \draw (3) to [out=110,in=50,looseness=8] (3) {};
    \draw (4) to [out=150,in=90,looseness=8] (4) {};
    \draw (5) to [out=190,in=130,looseness=8] (5) {};
    \draw (6) to [out=230,in=170,looseness=8] (6) {};
    \draw (7) to [out=270,in=210,looseness=8] (7) {};
    \draw (8) to [out=310,in=250,looseness=8] (8) {};
    \draw (9) to [out=350,in=290,looseness=8] (9) {};
    
    \end{tikzpicture}
    \caption{
    The network structure for experiments. The dark node at the center is the absorbing node $*$, and nodes labelled with numbers are transient nodes. Nodes without edges connecting them visits each other with zero probability. 
    \label{fig:network}}
\end{figure}

\section{Experiments}
\label{sec:exp}
We deploy our algorithms on a problem with 9 transient nodes. The results for Algorithm \ref{alg:adv} is in Figure \ref{fig:adv}, and the results for Algorithm \ref{alg:opt} is deferred to the Appendix. The evaluation of Algorithm \ref{alg:adv} is performed on a problem instance with the transition matrix specified in (\ref{eq:exp-adv}). 
    \begin{align}
    m_{ij} &= 
    \begin{cases}
        0.3, \text{ if } i = j ,  \\
        0.1, \text{ if } i = j \pm 1 \mod 9,  \\  
        0, \text{ otherwise.} 
    \end{cases}  \label{eq:exp-adv}
    \end{align} 
    
    The edge lengths are sampled from Gaussian distributions and truncated to between 0 and 1. Specifically, for all $t = 1,2,\cdots, T$, 
    $l_{ij}^{(t)} =
    \begin{cases}
    \clip_{[0,1]} \( W_t + 0.5 \), \text{ if } i = 0 \text{ and } j = * \\ 
    \clip_{[0,1]} \( W_t \), \text{ if } i \neq 0 \text{ and } j = * \\ 
    1 , \text{ otherwise}
    \end{cases} , $ 
    where $\clip_{[0,1]} (z) $ takes a number $z$ and clips it to $[0,1]$, $W_t \overset{i.i.d.}{\sim} \mathcal{ N } \( 0.5, 0.1 \)$, and $*$ denotes the absorbing node. 
    
\begin{figure}[h!]
    \centering
    \includegraphics[scale = 0.7]{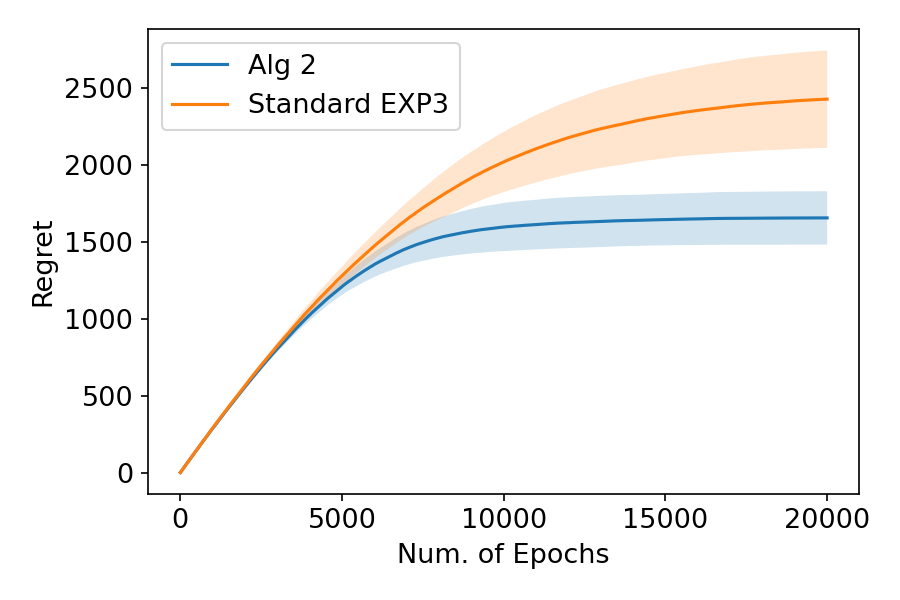}
    \caption{Experimental Results for Algorithm \ref{alg:adv}, compared with the standard EXP3 algorithm. Each solid line plot is an average of 10 runs. The shaded area indicate one standard deviation below and above the average.  \label{fig:adv}}
    
\end{figure} 

As shown in Figure \ref{fig:adv}, EXP3 algorithm performs better when using the, which is consistent with the guarantee provided in Theorem \ref{thm:adv}. For implementation purpose and a fair comparison, the estimator for Algorithm \ref{alg:adv} is $ \wh{Z}_{t,i} = \frac{  {Z}_{t,i}  \I_{ \[ i \in \P_{t,J_{t}} \] } }{ p_{ti} + \sum_{j \neq i} \wh{q}_{t,ji} p_{tj}  } $, and the estimators for EXP3 uses $ \wh{Z}_{t,i} = \frac{  {Z}_{t,i}  \I_{ \[ i = J_t \] } }{ p_{ti}  } $. The learning rate for both algorithms is set to $0.001$.



\section{Conclusion}


In this paper, we study the bandit problem where the feedback is random walk trajectories. This problem is motivated by influence maximization in computational social science. We show that, despite substantially more information can be extracted from the random walk trajectories, such problems are not significantly easier than its standard MAB counterpart in a mini-max sense. Behaviors of UCB and EXP3 are studied.

\bibliographystyle{apalike} 
\bibliography{reference} 

\appendix

\newpage
\onecolumn

\section{Proof Details of Theorems \ref{thm:K-lower}}


Construct $K+1$ instances specified by graphs: $\J_0 = \( G_t^{0} \)_{t=1}^T, \J_1 = \( G_t^{1} \)_{t=1}^T, \J_2 = \( G_t^{2} \)_{t=1}^T, \cdots, \J_K = \( G_t^{K} \)_{t=1}^T $. For any $t,k$, let all transition probabilities equal $p$. For any $t,k$, let the graph $ G_t^{k} $ have edge lengths $ \( \{ l_{i*}^{(t)} \}_{i \in [K]}, \{ l_{ij}^{(t)} \}_{i,j \in [K]} \) $ chosen randomly. In $\J_0$, all edge lengths are independently sampled from $ \text{Bernoulli} \( \frac{1}{2} \) $. In $\J_k$ ($k\ge 1$), $l_{k*}^{(t)}$ (for all $t \in [T]$, $k \ge 1$) are sampled from $ \text{Bernoulli} \( \frac{1}{2} + \frac{ \epsilon }{1- Kp}\) $, and all other edge lengths are sampled from $ \text{Bernoulli} \( \frac{1}{2} \) $. The proof is presented in three steps: Step 1 computes the KL-divergence using similar arguments discussed in the main text. Steps 2 \& 3 use standard argument for general lower bound proofs. 

\textbf{Step 1: compute the KL-divergence between $\J_0$ and $\J_k$}. 
Let $\Q_{t,j}^{(k)}$ span the probability of playing $ j $ at $t$ in instance $\J_k$. 
By chain rule, we have, for any $k = 2,\cdots, K$, 
\begin{align} 
    D_{KL} \( \Pr_{\J_0,\pi} \| \Pr_{\J_k,\pi} \) = \sum_{t=1}^T \sum_{j \in [K]} \Pr_{\J_0,\pi} \( J_t = j \) D_{KL} \( \Q_{t, j}^{(0)} \| \Q_{t, j}^{(k)} \). \label{eq:adv-kl-decomp} 
\end{align} 

Let $ X_0, L_1, X_1, L_2, \cdots $ be the nodes and edge length of each step in the trajectory after playing a node. The sample space of $\Q_{t, j}^{(k)}$ is spanned by $ X_0, L_1, X_1, L_2, \cdots $. 

By Markov property, we have, for all $i,j \in [K]$ and $s \in \mathbb{N}_+$, 
\begin{align} 
    &\Q_{t,i}^{(k)} \( L_{s+1} , X_{s+1}, L_{s+2}, X_{s+2}, \dots \rvert X_s = j \) = \Q_{t,j}^{(k)} , \quad \forall k = 0,1,2,\cdots, K.  \nonumber 
\end{align} 

Thus by applying the chain rule to $\Q_{t,j}^{(k)}$, we have 
\begin{align} 
    &D_{KL} \( \Q_{t,i}^{(0)} \| \Q_{t,i}^{(k)} \) \nonumber \\ 
    =& D_{KL} \( \Q_{t,i}^{(0)} (X_1, L_1) \| \Q_{t,i}^{(k)} (X_1, L_1) \) \nonumber \\ 
    &+ \sum_{x \in [K]} \sum_{l \in \{ 0 , 1 \} } \Pr \( X_1 =  x, L_1  =  l \)  \nonumber \\ 
    &\cdot D_{KL}  \( \Q_{t,i}^{(0)} (X_2, L_2,  ... \rvert X_1 = x, L_1  =  l ) \|  \Q_{t,i}^{(k)} ( X_2,  L_2,  ... \rvert X_1 =  x, L_1 =  l ) \) .  \label{eq:recurse-kl-adv-0} 
\end{align} 

Recall that the edge lengths are selected independent of other randomness. Thus we have 
\begin{align*} 
    \Q_{t,i}^{(k)} (X_2, \hspace*{-2pt}  L_2, \hspace*{-2pt}  ... \rvert X_1 \hspace*{-2pt}  = \hspace*{-2pt} x, L_1 \hspace*{-2pt} = \hspace*{-2pt} l ) 
    = 
    \Q_{t,i}^{(k)} (X_2, \hspace*{-2pt}  L_2, \hspace*{-2pt}  ... \rvert X_1 \hspace*{-2pt}  = \hspace*{-2pt} x ) , \quad k = 0,1,2, \cdots, K. 
\end{align*}

Thus (\ref{eq:recurse-kl-adv-0}) simplifies to 
\begin{align} 
    & \;
    D_{KL} \( \Q_{t,i}^{(0)} \| \Q_{t,i}^{(k)} \) \nonumber \\
    =&\;  
    D_{KL} \( \Q_{t,i}^{(0)} (X_1, L_1) \| \Q_{t,i}^{(k)} (X_1, L_1) \) + \sum_{j \in [K]} m_{ji} D_{KL} \( \Q_{t,j}^{(0)}  \| \Q_{t,j}^{(k)} \), 
    \label{eq:recurse-kl-adv}  
\end{align} 
where $m_{ji}$ ($m_{ji} = p$) is the probability of visiting $ j $ from $i$.

We have, for $ k\ge 1$, 
\begin{align}
    &\; 
    D_{KL} \( \Q_{t,i}^{(0)} (X_1, L_1) \| \Q_{t,i}^{(k)} (X_1, L_1) \) \nonumber \\ 
    =& \; 
    \begin{cases} 
        \frac{ 1 - Kp }{2} \log \frac{ (1-Kp) / 2 }{ (1-Kp)  \( \frac{1}{2} +  \frac{\epsilon}{ 1 - Kp } \) } + \frac{ 1 - Kp }{2} \log \frac{ (1-Kp) / 2 }{ (1-Kp)  \( \frac{1}{2} - \frac{\epsilon}{ 1 - Kp } \) }, & \text{if } i = k, \\
        0,  &\text{otherwise}. 
    \end{cases} 
    \nonumber \\ 
    \le& \; 
    \begin{cases} 
        \frac{\epsilon^2}{1 - Kp}, & \text{if } i = k, \\ 
        0,  &\text{otherwise}. 
    \end{cases} 
    \label{eq:recurse-kl-adv2} 
    %
\end{align} 
where the last line uses that $ x - \frac{x^2}{2} \le \log (1 + x) \le x $ for all $x \in [0,1]$. 


Combining (\ref{eq:recurse-kl-adv}) and (\ref{eq:recurse-kl-adv2}) gives
\begin{align}
    D_{KL} \( \Q_{t,i}^{(0)} \| \Q_{t,i}^{(k)}  \) 
    \le& \; 
    \begin{cases} 
        \frac{\epsilon^2}{1 - Kp}  + \frac{ p \epsilon^2 }{ (1 - Kp)^2 }, & \text{if } i = k, \\ 
        \frac{ p \epsilon^2 }{ (1 - Kp)^2 },  &\text{otherwise}. 
    \end{cases} 
\end{align}


Plugging the above results into (\ref{eq:adv-kl-decomp}) and we get
\begin{align*}
    D_{KL} \( \Pr_{\J_0 , \pi} \| \Pr_{\J_k , \pi} \) 
    =& \;   
    \sum_{t=1}^T \sum_{j \in [K]}  \Pr_{\J_0,\pi} \( J_t = j \) D_{KL} \( \Q_{t, j}^{(0)} \| \Q_{t, j}^{(k)} \) \\ 
    =& \;  
    \sum_j \E_{\J_0 , \pi} \[ N_j \]  D_{KL} \( \Q_{t, j}^{(0)} \| \Q_{t, j}^{(k)} \) \\ 
    \le& \; 
    \frac{\epsilon^2 \E_{\J_0 , \pi} \[ N_k \] }{ 1 - Kp } + \frac{\epsilon^2 p T }{ ( 1 - Kp )^2 }, 
\end{align*}
where $N_j$ is the number of times $j$ is played.

\textbf{Step 2: compute the optimality gap between nodes}. 

Let $H_k$ be the vector of (expected) hitting times in instance $\J_k$. The vector $H_k$ $(k \ge 2)$ satisfies
\begin{align*}
    H_k = M H_k + (1 - Kp) \( \frac{1}{2} \bm{1} + \frac{\epsilon}{1 - Kp} \bm{e}_k  \),  
\end{align*}
where $M$ is the transition matrix among transient nodes (all entries of $M$ are $p$), $\bm{1}$ is the all $\bm{1}$ vector and $\bm{e}_k$ is the $k$-th canonical basis vector. 

Solving the above equation gives, for $k \ge 2$
\begin{align*} 
    H_k = (1 - Kp) \( I + \frac{M}{1 - Kp} \) \( \frac{1}{2} \bm{1} + \frac{ \epsilon}{1 - Kp} \bm{e}_k  \) . 
\end{align*}

Thus the optimality gap, which is the difference between the hitting time of node $k$ in $\J_k$ and the hitting time of other nodes in $\J_k$ $(k \ge 1)$, is $ \Delta := \epsilon $. 


\textbf{Step 3: apply Yao's principle and Pinsker's inequality to finish the proof}.

By Pinsker's inequality, we have $\forall j,k, $ 
\begin{align} 
    \rvert \Pr_{\J_0 , \pi} ( J_t = j ) - \Pr_{\J_k , \pi} ( J_t = j ) \rvert
    \le& \;  
    \sqrt{2 D_{KL} \( \Pr_{\J_0 , \pi} \| \Pr_{\J_k , \pi} \) } .  \label{eq:use-pinsker} 
\end{align} 

Thus for the regret against $k$ is instance $\J_k$, we have 
\begin{align} 
    &\; \frac{1}{K} \sum_{k=1}^K \E_{\J_k, \pi } \[ \Regadv_k (T) \] \nonumber \\ 
    =& \; 
    \frac{1}{K} \sum_{k=1}^K \sum_t \E_{\J_k, \pi} \[ Y_{k,t } \] - \E_{\J_k, \pi} \[ Y_{ J_t,t } \] \nonumber \\ 
    =& \; 
    \frac{ \epsilon }{K} \sum_{k=1}^K \sum_t \Pr_{\J_k, \pi} ( J_t \neq k ) \tag{by the Wald's indentity} \\ 
    =& \; 
    \epsilon T - \frac{\epsilon }{K} \sum_{k=1}^K \sum_t  \Pr_{\J_k, \pi} \( J_t = k \) \nonumber \\ 
    \ge& \; 
    \epsilon T - \frac{\epsilon }{K} \sum_{k=1}^K \sum_t \Pr_{\J_0, \pi} ( J_t = k ) - \frac{\epsilon }{K} \sum_{k=1}^K \sum_t \sqrt{ 2 D_{KL} \( \P_{\J_0, \pi } \| \P_{\J_k, \pi } \) } \tag{by Eq. \ref{eq:use-pinsker}} \\ 
    \ge& \; 
    \frac{(K-1) \epsilon T }{K} - \epsilon^2 T \sqrt{ \frac{1}{K} \sum_{k=1}^K \( \frac{ 2 \E_{\J_0, \pi} \[ N_k \] }{1 - Kp} + \frac{2 p T }{ (1 - Kp)^2 } \) } \tag{by Jensen's inequality} \\ 
    =& \; 
    \frac{(K-1) \epsilon T }{K} - \epsilon^2 T \sqrt{ \frac{ 2 T }{K ( 1 - Kp ) } + \frac{2 p T }{ (1 - Kp)^2 } } . \label{eq:tag1}
\end{align} 

Now we set $p = \frac{1}{2K}$ and $ \epsilon = \frac{1}{4 \sqrt{2}} \frac{K-1}{ {K}} \sqrt{ \frac{ K }{T} } $ and (\ref{eq:tag1}) gives 
\begin{align*} 
    \frac{1}{K} \sum_{k=1}^K \E_{\J_k, \pi } \[ \Regadv_k (T) \] 
    \ge 
    \frac{1}{8\sqrt{2}} \sqrt{KT} 
\end{align*} 

Thus we have 
\begin{align*}
    \max_k \E_{\J_k, \pi} \[ \Regadv_k (T) \] 
    \ge& \;  
    \frac{1}{K} \sum_{k=1}^K \E_{\J_k, \pi} \[ \Regadv_k (T) \] \\ 
    \ge& \; 
    \frac{1}{8\sqrt{2}} \sqrt{KT} . 
\end{align*}


\section{Proof Details for Section \ref{sec:algs}}
\label{app:pf-adv}

We use the following notations for simplicity. 
	1. We write $\wt{p}_{tj} = p_{tj} + \sum_{i\neq j} q_{ij} p_{ti} $, and $\wh{p}_{tj} = p_{tj} + \sum_{i\neq j} \wh{q}_{t,ij} p_{ti} $. 
	2. We use $\mathcal{F}_t$ to denote the $\sigma$-algebra generated by all randomness up to the end of epoch $ t$. 
		We use $ \mathcal{F}_{t,i} $ to denote the $\sigma$-algebra of all randomness up to the first occurrence of node $i$ in epoch $t$ (or end of epoch $t$ if $i$ is not visited in epoch $t$). We use $\E_t$ to denote the expectation conditioning on $\mathcal{F}_t$, i.e., $\E_t \[ \cdot \] = \E \[ \cdot \rvert \mathcal{F}_t \]$. 
	3. Unless otherwise stated, we use $\sum_t$ and $\sum_j$ as shorthand for $\sum_{t=1}^T$ and $\sum_{j \in [K]}$, respectively.

In Appendix \ref{app:add}, we state some preparation properties needed for proving Lemma \ref{lem:bound}. Proposition \ref{prop:quick-bound} is also included in this part. In Appendix \ref{app:bound}, we provide a proof for Lemma \ref{lem:bound}.

\subsection{Additional Properties}
\label{app:add}

As Proposition \ref{prop:est} suggests, number of times a node is visited linearly accumulates with number of epochs $ t $. 
We state this observation below in Lemma \ref{lem:event}. 
\begin{lemma} 
\label{lem:event}
For $v \in V$ and $t$, it holds that
    $\Pr \Big( {N}_t^+ (v) - {N}_t (v) -  \alpha_{v} \( t -  {N}_t (v) \) \ge - \lambda \Big) \le e^ { - \frac{ \lambda^2}{2 t} } . $ 
\end{lemma} 
\begin{proof}
    Recall $\P_{t,J_t} = \( X_{t,0}, L_{t,1}, X_{t,1}, L_{t,2}, X_{t,2}, \dots,  L_{t,H_{t,J_t}}, X_{t,H_{t,J_t}} \)$ is the trajectory for epoch $t$ and $ X_{i,0} $ is the node played at epoch $i$. For a fixed node $v \in V$, consider the random variables $ \left\{ \I_{ \[ v \in \P_{t, J_t} \setminus \{ X_{t,0} \} \] } \right\}_t $, which is the indicator that takes value 1 when $v$ is covered in path $\P_{t,J_t}$ but is not played at $t$. 
    From this definition, we have 
    \begin{align*}
        \sum_{k=1}^t \I_{ \[ v \in \P_{k, J_k} \setminus \{ X_{k,0} \} \] } = {N}_t^+ (v) - {N}_t (v).
    \end{align*}
    From definition of $ \alpha_v $, we have 
    \begin{align*}
        \E \[ {N}_t^+ (v) - {N}_t (v) \] =  \E \[ \sum_{k=1}^t \I_{ \[ v \in \P_i \setminus \{ X_{k,0} \} \] } \] 
    \ge \alpha_{v} \( t - \E \[ {N}_t (v) \] \) .
    \end{align*}
    
    Thus by one-sided Azuma's inequality, we have for any $\lambda > 0$, 
    \begin{align}
        \Pr \Big( {N}_t^+ (v) - {N}_t (v) -  \alpha_{v} \( t -  {N}_t (v) \) \ge - \lambda \Big) \le \exp \( - \frac{ \lambda^2 }{2 t } \). 
    \end{align}
\end{proof}

\begin{lemma}
	\label{lem:q-mean-var}
	For any $t, i,j$, it holds that 
		$ \mathbb{V} \[ \wh{q}_{t,ij} \] = \Ot \( \frac{1}{t} \).$ 
\end{lemma} 
\begin{proof}
    For the variance, we have 
	\begin{align*}
		\mathbb{V} \[ \wh{q}_{t,ij} \] 
		=& \; 
		\sum_{m=1}^t \mathbb{V} \[ \wh{q}_{t,ij} \bigg\rvert N_t^+ (i) = m \] \Pr \( N_t^+ (i) = m  \) \\
		\le& \;  
		\sum_{m=1}^t \frac{1}{m} \Pr \( N_t^+ (i) = m  \) = \E \[ \frac{1}{N_t^+ (i)} \]. 
	\end{align*}
	By Lemma \ref{lem:event} and a union bound, we know, for any $\delta \in (0,1)$, 
	\begin{align*}
		\Pr \Big( {N}_t^+ (i) \ge \alpha t - \sqrt{2t\log (2TK/\delta)}, \quad \forall i \in [K], \; t \in [T] \Big) \le {\delta}. 
	\end{align*} 
	Thus it holds that
	\begin{align*}
		\E \[ \frac{1}{N_t^+ (i)} \] 
		=& \; 
		\E \[ \frac{1}{N_t^+ (i)} \big\rvert {N_t^+ (i)} \ge \alpha t - \sqrt{2t\log (TK/\delta)} \] \\
		&\quad \cdot 
		\Pr \( N_t^+ (i) \ge \alpha t - \sqrt{2t\log (TK/\delta)} \) \\
		&+ \E \[ \frac{1}{N_t^+ (i)} \big\rvert {N_t^+ (i)} < \alpha t - \sqrt{2t\log (TK/\delta)} \] \\
		&\quad \cdot \Pr \( N_t^+ (i) < \alpha t - \sqrt{2t\log (TK/\delta)} \) \\
		\le& \; 
		\frac{1}{ \max \left\{ 1, \alpha t - \sqrt{2t\log (TK/\delta)} \right\} } + \delta . 
	\end{align*}
	Setting $\delta = \frac{1}{t}$ concludes the proof. 
\end{proof} 

Next, we consider a high probability event and approximate $\wh{q}_{t,ij}$ under this event. 

\begin{lemma}
    \label{lem:concen-q}
    For any $\epsilon \in (0,1)$, let 
    \begin{align*}
        \mathcal{E}_T'  \hspace*{-2pt} :=  \hspace*{-2pt} \Bigg\{  & \left\rvert \wh{q}_{t,ij} - q_{ij} \right\rvert \le \sqrt{ \frac{ 2 \mathbb{V} \[ \wh{q}_{t,ij} \] \log (T/\epsilon) }{ N_{t-1}^+ (i) } } + \frac{ \log (T/\epsilon) }{ 3 N_{t-1}^+ (i) } , \\
        & \quad \; N_t^+ (i)   \hspace*{-2pt} \ge  \hspace*{-2pt} \alpha t \hspace*{-2pt} - \hspace*{-2pt} \sqrt{ t \log (TK/\epsilon) }, \; \forall i,j  \hspace*{-2pt} \in  \hspace*{-2pt} [K] \forall t \in [T] \Bigg\} .
    \end{align*} 
    It holds that $\Pr \( \mathcal{E}_t' \) \ge 1 - \frac{ 2 \epsilon}{T} $ and under $ \mathcal{E}_T' $, for all $i,j \in [K]$ and $t \in [T]$, 
    \begin{align}
        \wh{q}_{t,ij} = q_{ij} \pm \mathcal{O} \( { \frac{ \log (TK/\epsilon ) }{  t} } \) , 
        \text{ and } 
        \wh{p}_{ti} = \wt{p}_{ti} \pm \mathcal{O} \( { \frac{ \log (TK/\epsilon ) }{  t} } \) . 
    \end{align} 
\end{lemma}

\begin{proof}
    By Bennett's inequality, it holds that 
	\begin{align}
		\Pr \( \left\rvert \wh{q}_{t,ij} - q_{ij} \right\rvert \ge \sqrt{ \frac{ 2 \mathbb{V} \[ \wh{q}_{t,ij} \] \log (KT/\epsilon) }{ N_{t-1}^+ (i) } } + \frac{ \log (KT/\epsilon) }{ 3 N_{t-1}^+ (i) } \) \le \epsilon
	\end{align} 
	By Lemma \ref{lem:event} and a union bound, we known $\Pr \( \mathcal{E}_T' \) \ge 1 - 2 \epsilon$. By Lemma \ref{lem:q-mean-var}, we know that $\mathbb{V} \[ \wh{q}_{t,ij} \] = \Ot \( \frac{1}{t} \)$. Thus, under event $\mathcal{E}_T'$, it holds that 
	\begin{align*} 
		\wh{q}_{t,ij} = q_{ij} \pm \Ot \(  \frac{  \log \( T / \epsilon \) }{t } \) , 
	\end{align*} 
	and thus  
	\begin{align*} 
	    \wh{p}_{ti} = \wt{p}_{ti} \pm \Ot \(  \frac{  \log \( T / \epsilon \) }{t } \) .  
	\end{align*} 
\end{proof}

\begin{lemma} 
    \label{lem:e-false} 
	For any $B$, let 
	$ \mathcal{E}_T (B) := \left\{ {Z}_{t,j} \le B \text{ for all } t = 1,2,\cdots, T, \text{ and } j \in [K] \right\}. $ 
	For any $\epsilon \in (0,1)$ and $B = \advB$, it holds that 
	\begin{align*} 
		&\Pr \( \mathcal{E}_T \( B \)\) \ge 1 - \epsilon, \\ 
		&\E \[ {Z}_{t,i} \rvert \text{not }  \mathcal{E}_T (B) \] 
    	\le  KB + \frac{K}{(1 - \rho)^2} , \\
		&\E \[ {Z}_{t,i}^2 \rvert \text{not }  \mathcal{E}_T (B) \] 
    	\le  KB^2 + \frac{2 K}{(1 - \rho)^3} . 
	\end{align*} 
\end{lemma}

\begin{proof}
	Since all edge lengths are smaller than 1, we have, for any integer $B$, 
    \begin{align} 
        \Pr \( Z_{t,i} > B \) 
        &\le \Pr \( \{ \text{random walk starting from $i$ does not terminate in $B$ steps} \} \) \nonumber \\
        &= \sum_{l=B}^\infty \Pr \( \{ \text{random walk starting from $i$ terminates at step $l$} \} \) \nonumber \\
        &= \sum_{l=B}^\infty \sum_{ j = 1 }^K \[ M^l \]_{ij}  \le \sum_{l=B}^\infty \rvert M \rvert_\infty^l \le \sum_{l=B}^\infty \rho^l \le \frac{ \rho^B }{1 - \rho}. \nonumber
    \end{align} 
    Thus with probability at least $ 1 - \frac{ \rho^B }{1 - \rho} $, we have $Z_{t,i} \le B$. We define
    \begin{align}
        \mathcal{E}_T (B) := \left\{ {Z}_{t,j} \le B \text{ for all } t = 1,2,\cdots, T, \text{ and } j \in [K] \right\}. \nonumber
    \end{align}
    By a union bound, $ \Pr \( \mathcal{E}_T (B) \) \ge 1 - \frac{ KT \rho^B }{ 1 - \rho } $. Now we can set $B = \advB $ so that $ \Pr \( \mathcal{E}_T (B)\) \ge 1 - \epsilon $. 
    
    The random variables $Z_{t,i}$ also has the memorylessness-type property: 
    \begin{align*}
    	& \; \E \[ Z_{t,i} \rvert \text{not } \mathcal{E}_T (B) \] \\
    	\le& \E \[ Z_{t,i} \rvert Z_{t,i} > B \] \le \sum_{l=B+1}^\infty l \frac{ \Pr \( \{ \text{$\P_{t,i}$ terminates at step $l$} \} \cap \{  Z_{t,i} > B \} \) }{ \Pr \( Z_{t,i} > B \) }  \\
    	=& \; 
    	\sum_{l=B+1}^\infty l \frac{ \sum_{j} \Pr \( \{ \text{$\P_{t,i}$ terminates at step $l$} \} \cap \{ \text{the $(B+1)$-th step is at $j$} \} \) }{ \sum_{j} \Pr \( \{ \text{the $(B+1)$-th step is at $j$} \} \) } \\
    	\le& \; 
    	\sum_{l=B+1}^\infty l \sum_j \Pr \( \{ \text{$\P_{t,i}$ terminates at step $l$} \} \big\rvert \{ \text{the $(B+1)$-th step is at $j$} \} \) \\
    	=& \; \sum_{j} \sum_{l=1}^\infty (l + B ) \Pr \( \{ \text{$\P_{t,j}$ terminates at step $l$} \} \) \\
    	=& \; \sum_j \E \[ Z_{t,j} + B\] \le KB + \sum_j \E \[ Z_{t,j} \]
    \end{align*}
    where we use Markov property on the second last line. 
    
    Since $ \E \[ Z_{t,j} \] \le \mathcal{O} \( \frac{1}{ \(1 - \rho \)^2 } \) $, we insert this into the above equation to get
    \begin{align*}
    	\E \[ Z_{t,i} \rvert \text{not } \mathcal{E}_T (B) \]  \le \mathcal{O} \( KB + \frac{K}{(1 - \rho)^2} \) .
    \end{align*}
    
    Similarly, we have 
    \begin{align*}
    	\E \[ Z_{t,i}^2 \rvert \text{not } \mathcal{E}_T (B) \] 
    	\le \mathcal{O} \( KB^2 + \frac{K}{(1-\rho)^3} \). 
    \end{align*}
\end{proof}

\begin{remark}
    By Lemmas \ref{lem:concen-q} and \ref{lem:e-false}, we know $ \wh{p}_{ti} = \wt{p}_{ti} + \Ot \( \frac{1}{t} \) $ and $ Z_{t,i} \le B $ ($\forall i \in [K], t \in [T]$) hold with high probability. Let $ \Sigma $ be the whole event space spanned by all possibilities of $T$ rounds of plays. From now on, we will restrict our attention to the event space $ \Sigma' = \{ e \in \Sigma : e \cap \mathcal{E}_T (B) \cap \mathcal{E}_T' \} $, and work in this event space unless otherwise noted. 
\end{remark}

\begin{proposition}
    \label{prop:quick-bound}
    Fix any $a \in (0,1]$. We have
    \begin{align}
        \frac{x}{ x + (1-a)x } \le x + \frac{1 - \sqrt{a}}{ 1 + \sqrt{a} } , \qquad \forall x \in (0,1). 
    \end{align}
\end{proposition} 

\begin{proof} 
    If suffices to show, for any $a \in (0,1]$, the function $f_a (x) := \frac{x}{x + (1-x)a } - x $ is upper bounded by $ \frac{1 - \sqrt{a}}{ 1 + \sqrt{a} } $. This can be shown via a quick first-order test. At $ x_{\max} = \frac{\sqrt{a}}{1 + \sqrt{a}} $, the maximum of $f_a $ is achieved, and $ f_a (x_{\max}) = \frac{1 - \sqrt{a}}{ 1 + \sqrt{a} } $.  
\end{proof}

\subsection{Proof of Lemma \ref{lem:bound}}
    \label{app:bound}

	By law of total expectation, we have 
    \begin{align}
    	\E_{t-1} \[ Z_{t,i} \I_{ \[ i \in \P_{t,J_t} \] } \] 
    	=& \;  
    	\E_{t-1} \[ \E \[ Z_{t,i} \I_{ \[ i \in \P_{t,J_t} \] } \big\rvert \mathcal{F}_{t,i} \] \] \nonumber \\ 
    	=& \;  
    	\E_{t-1} \[ l_{t,i} \I_{ \[ i \in \P_{t,J_t} \] } \] = l_{t,i} \wt{p}_{ti}. 
        \nonumber 
    \end{align}
    
	Also, we have 
    \begin{align}
    	&\E_{t-1} \[ \( \frac{  Z_{t,i} - B }{B} \)^2 \] \le 1 
    	\label{eq:sec-moment}
    \end{align}
    
    By Lemma \ref{lem:concen-q}, we know
    \begin{align} 
    	\frac{1}{ \wh{p}_{tj} } = \frac{1}{ \wt{p}_{tj} } \pm \Ot \( \frac{\log \( T/ \epsilon \)}{t} \), \quad \forall t \in [T],j \in [K] \label{eq:expand-1}
    \end{align} 
    
    We can use (\ref{eq:expand-1}) to get 
    \begin{align}
        & \E_{t-1} \[ \exp \( \frac{ \beta }{B} \( l_{t,i} - B \) -  \beta  \( \frac{ \frac{ {Z}_{t,i} - B }{ B } \I_{ \[ i \in \P_{t,J_{t}} \] } + \beta }{  p_{ti} + \sum_{j \neq i} \wh{q}_{t,ji} p_{tj}   } \) \)  \] \nonumber \\
        {=}& \E_{t-1} \hspace*{-2pt} \[ \hspace*{-1pt} \exp \hspace*{-2pt} \( \hspace*{-2pt} \frac{ \beta }{B} \( l_{t,i} \hspace*{-2pt} - \hspace*{-2pt} B\) - \hspace*{-2pt} \beta  \( \frac{ \frac{ {Z}_{t,i} - B }{ B } \I_{ \[ i \in \P_{t,J_{t}} \] } + \beta }{  \wt{p}_{ti}  } \)  \hspace*{-2pt} + \hspace*{-2pt} \Ot \hspace*{-2pt} \( \frac{\log \( T/ \epsilon \)}{t} \) \hspace*{-4pt} \) \hspace*{-4pt} \] \nonumber \\ 
        {=}& \exp \( - \frac{\beta^2 }{ \wt{p}_{ti} } \hspace*{-2pt} + \hspace*{-2pt} \Ot \( \frac{\log \( T/ \epsilon \)}{t} \) \) \cdot \nonumber \\
        &\qquad \qquad \E_{t-1} \[ \exp \( \frac{ \beta }{B} \( l_{t,i} - B\) -  \beta  \( \frac{ \frac{ {Z}_{t,i} - B }{ B } \I_{ \[ i \in \P_{t,J_{t}} \] } }{  \wt{p}_{ti}  } \)  \)  \] \label{eq:lem-hp-1} . 
     \end{align}

     Since $ \exp(x) \le 1 + x + x^2 $ for $x \le 1$, we have 
     \begin{align}
     	& \E_{t-1} \[ \exp \( \frac{ \beta }{B} \( l_{t,i} - B\) -  \beta  \( \frac{ \frac{ {Z}_{t,i} - B }{ B } \I_{ \[ i \in \P_{t,J_{t}} \] } }{  \wt{p}_{ti}  } \) \) \] \nonumber \\
     	\le& 
     	 1 + \E_{t-1} \[  \frac{ \beta }{B} \( l_{t,i} - B \) -  \beta  \( \frac{ \frac{ {Z}_{t,i} - B }{ B } \I_{ \[ i \in \P_{t,J_{t}} \] } }{  \wt{p}_{ti}  } \)  \] \nonumber \\
     	 &\qquad + \E_{t-1} \[  \( \frac{ \beta }{B} \( l_{t,i} - B \) -  \beta  \( \frac{ \frac{ {Z}_{t,i} - B }{ B } \I_{ \[ i \in \P_{t,J_{t}} \] } }{  \wt{p}_{ti}  } \)  \)^2  \] . \label{eq:lem-hp-2} 
     \end{align} 
     
     Since $ \E_{t-1} \[  \I_{ \[ i \in \P_{t,J_t} \] } \] = \wt{p}_{tj} $, $ \E_{t-1} \[ \( {Z_{t,i} - B} \) \I_{ \[ i \in \P_{t,J_t} \] } \] = \( l_{t,j} - B \) \wt{p}_{tj} $ and $ \E_{t-1} \[ \( \frac{Z_{t,i} - B}{B} \)^2 \I_{ \[ i \in \P_{t,J_t} \] } \] \le \wt{p}_{tj} $, we have 
     \begin{align}
     	&1 + \E_{t-1} \[  \frac{ \beta }{B} \( l_{t,i} - B \) -  \beta  \( \frac{ \frac{ {Z}_{t,i} - B }{ B } \I_{ \[ i \in \P_{t,J_{t}} \] } }{  \wt{p}_{ti}  } \)  \] \nonumber \\
     	&\qquad \qquad + \E_{t-1} \[  \( \frac{ \beta }{B} \( l_{t,i} - B \) -  \beta  \( \frac{ \frac{ {Z}_{t,i} - B }{ B } \I_{ \[ i \in \P_{t,J_{t}} \] } }{  \wt{p}_{ti}  } \)  \)^2  \] \nonumber \\
     	=&
     	1  - \frac{ \beta^2 }{B^2} \( l_{t,i} - B \)^2  + 
     		  \frac{ \beta^2 }{  \wt{p}_{ti}  } 
     	\le 1  + \frac{ \beta^2 }{  \wt{p}_{ti}  } 
        \le 
     		  \exp \( \frac{ \beta^2 }{  \wt{p}_{ti}  } \) \label{eq:lem-hp-3},  
    \end{align}
    where the last inequality uses $1 + x \le \exp (x)$.

	We can now combine (\ref{eq:lem-hp-1}), (\ref{eq:lem-hp-2}) and (\ref{eq:lem-hp-3}) and get
	\begin{align*}
		&\E_{t-1} \[ \exp \( \frac{ \beta }{B} \( l_{t,i} - B \) -  \beta  \( \frac{ \frac{ {Z}_{t,i} - B }{ B } \I_{ \[ i \in \P_{t,J_{t}} \] } + \beta }{  p_{ti} + \sum_{j \neq i} \wh{q}_{t,ji} p_{tj}   } \) \) \]  \\
		\le& 
		\exp \( \Ot \( \frac{\log \( T/ \epsilon \)}{t} \) \)
	\end{align*}

    Let $X = \frac{ \beta }{B} \sum_{t=1}^T \( l_{t,i} - B \) -  \beta \sum_{t=1}^T  \( \frac{ \frac{ {Z}_{t,i} - B }{ B } \I_{ \[ i \in \P_{t,J_{t}} \] } + \beta }{  p_{ti} + \sum_{j \neq i} \wh{q}_{t,ji} p_{tj}   } \) $ for simplicity. Note that 
        $ X = \beta  \( \sum_{t=1}^T \frac{ l_{t,i} - B }{B} - \sum_{t=1}^T \wh{Z}_{t,i} \) $. 

    We combine (\ref{eq:lem-hp-1}), (\ref{eq:lem-hp-2}) and (\ref{eq:lem-hp-3}) and get
    \begin{align*}
    	\E \[ e^X  \] 
    	\le& 
    	\prod_{t=1}^T \exp \( \Ot \( \frac{\log \( T/ \epsilon \)}{t} \) \) 
    	= 
    	\Ot \(  T / \epsilon \)  . 
    \end{align*}

    By Markov inequality and the above results, we have 
    \begin{align*}  
        \Pr \(  \frac{X}{\beta} \ge \frac{\log \(  T/\epsilon^2 \) } {\beta} \) \le \Ot \(  \epsilon \)  , 
    \end{align*}
    which concludes the proof. 
    


    
\section{Additional Experimental Results} 

Algorithm \ref{alg:opt} is empirically studied here. In Figure \ref{fig:sto}, we plot both the regret and error of estimators of Algorithm \ref{alg:opt}. As a consequence of Lemma \ref{lem:event}, information about all node accumulates linearly, and the regret will stopped increasing after a certain point. Note that this does not contradict Theorem \ref{thm:lower} or Theorem \ref{thm:K-lower}. The reason is that this figure shows the regret and error of estimation for a given instance, whereas the lower bound theorems assert the existence of some instance (not the given instance) on which no algorithm can beat an $\Omega (\sqrt{T})$ lower bound. Since this some instance must satisfy Assumption \ref{assumption:transition}, we show that bandit problems with random walk feedback is not easier than their standard counterpart. The right subfigure in Figure \ref{fig:sto} shows that the estimation errors quickly drops to zero, which shows that the flat regret curve is a consequence of learning, not a consequence of luck. 

\begin{figure}
    \centering
    \includegraphics[width = 0.45\textwidth]{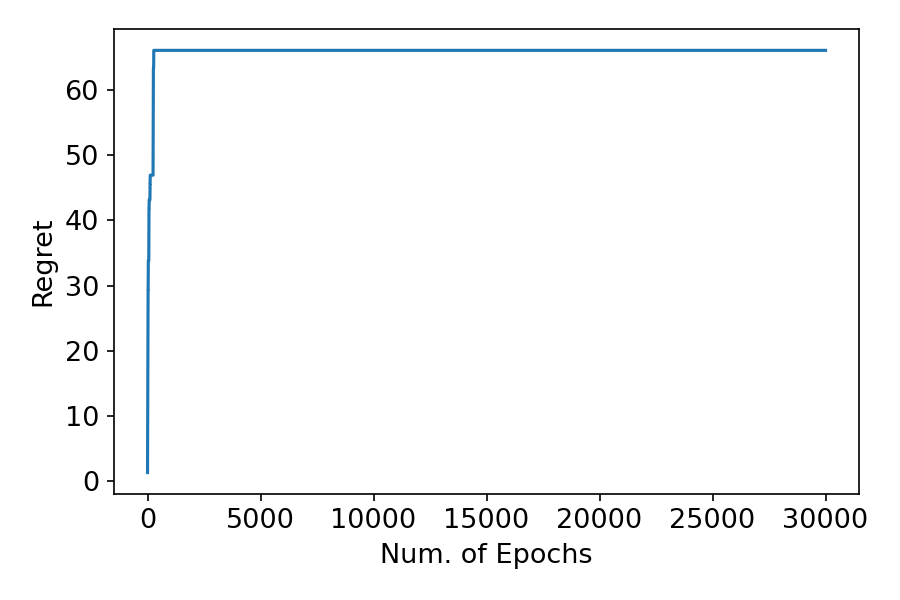} \quad  \includegraphics[width = 0.45\textwidth]{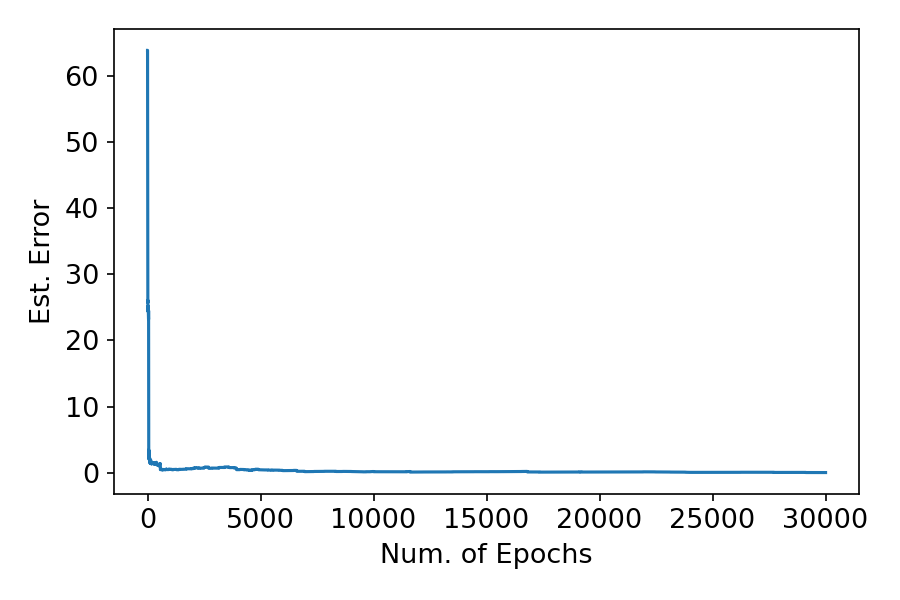} 
    \caption{Experimental results for Algorithm \ref{alg:opt}. \textit{Left:} The regret of Algorithm \ref{alg:opt} versus number of epochs. \textit{Right:} The estimation error of the hitting time estimators versus number of epochs. The solid line plot in both figures are averaged over 10 runs. \label{fig:sto}} 
    
\end{figure}

\end{document}